\documentclass{article}

\usepackage[preprint]{neurips_2026}

\usepackage[utf8]{inputenc} 
\usepackage[T1]{fontenc}    
\usepackage{hyperref}       
\usepackage{url}            
\usepackage{booktabs}       
\usepackage{amsfonts}       
\usepackage{nicefrac}       
\usepackage{microtype}      
\usepackage{xcolor}         
\usepackage{amsmath} 
\usepackage{graphicx}
\usepackage{algorithm}
\usepackage{algorithmic}
\usepackage{amsthm}
\usepackage{comment}
\usepackage{braket}

\newtheorem{proposition}{Proposition}
\newtheorem{remark}{Remark}

\newtheorem{theorem}{Theorem}
\newtheorem{lemma}{Lemma}

\title{Intrinsic Flow Matching on Quantum Pure-State Manifolds with Phase-Aligned Transport}

\author{%
  Jian Xu \\
  RIKEN iTHEMS \\
  RIKEN AIP \\
  \texttt{jian.xu@riken.jp} \\
  \And
  Delu Zeng \\
  South China University of Technology \\
  \And
  John Paisley \\
  Columbia University \\
  \And
  Qibin Zhao \\
  RIKEN AIP \\
}

\begin{document}

\maketitle

\begin{abstract}

Quantum pure-state ensembles live on complex projective space, making flat Euclidean generative modeling geometrically mismatched. We introduce \emph{Intrinsic Flow Matching} (IFM), a deterministic transport framework on \(\mathbb{CP}^{d-1}\) that learns tangent velocity fields using Pancharatnam phase-aligned conditional paths. IFM replaces local score teachers and reverse-time stochastic sampling with manifold probability flow, while horizontal parameterization removes redundant ambient directions. We show that the IFM objective recovers the induced marginal transport field, represents deterministic projective ensemble flows, and yields endpoint and stability guarantees. Empirically, IFM often improves over ambient Euclidean flow matching across higher-qubit, multimodal, spin-coherent, physics-inspired, and amplitude-encoded MNIST image-vector benchmarks, with strongest gains on high-dimensional and coherence-sensitive tasks but not uniformly across every metric.
\end{abstract}

\section{Introduction}

Generative modeling has become a powerful framework for high-dimensional data \citep{ho2020denoising,song2021scorebased,lipman2023flowmatchinggenerativemodeling}, and distributions over quantum pure states arise naturally as mathematical objects in many-body physics, quantum information, and quantum machine learning \citep{bengtsson2017geometry,biamonte2017qml,preskill2018nisq,schuld2019feature}.
Unlike Euclidean data, pure states are defined only up to global phase and live on the projective manifold $\mathbb{CP}^{d-1}$ \citep{mielnik1968geometry,bengtsson2017geometry,ogiue1969complex,wu2022metric}; ambient generative models therefore ignore gauge redundancy and may learn transport in directions with no physical meaning.

Recent intrinsic diffusion approaches, especially stochastic Schr\"odinger diffusion models (SSDM) \citep{Gisin1992TheQD,bouten2004stochastic,xu2026ssdm}, show that quantum-state generation can be posed directly on the pure-state manifold.
However, they still inherit diffusion machinery such as local score estimation, teacher approximations, and reverse-time stochastic simulation, which can become difficult as the Hilbert-space dimension grows.

We argue that quantum pure-state generation is better treated as intrinsic probability transport on $\mathbb{CP}^{d-1}$.
This leads to \emph{Intrinsic Flow Matching} (IFM), a flow-based framework that learns tangent velocity fields through a geometry-aware flow matching objective \citep{lipman2023flowmatchinggenerativemodeling,lipman2022flow,chen2023flow,tong2023conditional}.
IFM replaces local diffusion teachers and reverse stochastic sampling with direct velocity matching and deterministic manifold sampling, while its phase-aligned paths use the Pancharatnam in-phase convention \citep{pancharatnam1956generalized,garza2023deciphering,ferrer2023topological} to remove unphysical gauge rotation.
The resulting dynamics admit a probability-current interpretation on quantum state space \citep{delgado1999quantum,hodge2014electron,wyatt2005quantum,tsekov2009bohmian}.

Beyond the change of generative paradigm, our goal is to make the transport formulation mathematically explicit.
We show that the IFM population objective recovers the marginal transport field induced by a chosen path family, and that the corresponding marginal density evolves under the intrinsic continuity equation on $\mathbb{CP}^{d-1}$.
We further prove that the conditional IFM objective differs from its marginal regression counterpart only by an additive constant, paralleling the clean conditional-to-marginal reduction that underlies modern flow-matching theory.
We also show that deterministic tangent flows on projective state space induce exactly the type of ensemble dynamics modeled by IFM, giving a precise sense in which the framework is compatible with physics-inspired transport viewpoints based on probability currents and continuity laws.
Finally, under deterministic Monge couplings and Fubini--Study geodesic interpolation, IFM exactly recovers the displacement-interpolation velocity field, connecting our framework to geodesic and optimal-transport constructions on quantum state space \citep{bengtsson2017geometry,tong2023improving,chewi2021fast}.

A central question is whether the benefit comes merely from replacing diffusion with flow matching, or whether \emph{intrinsic geometry itself} matters.
To answer this, we distinguish between two possibilities:
performing flow matching intrinsically on $\mathbb{CP} ^{d-1}$, versus performing flow matching in an ambient Euclidean representation and projecting back to valid quantum states.
This comparison allows us to isolate the role of manifold geometry from the choice of generative paradigm.
Our experiments show that the advantage of IFM is not solely due to switching from diffusion to flow, but also due to respecting the intrinsic geometry of pure-state space.
This effect is especially clear on higher-qubit single-cluster scaling, controlled coherence-sensitive multimodal and spin-coherent benchmarks, and an amplitude-encoded MNIST image-vector benchmark.

Our contributions are summarized as follows:
\begin{itemize}
    \item We reformulate quantum pure-state generation as intrinsic probability transport on $\mathbb{CP}^{d-1}$ and propose \emph{Intrinsic Flow Matching} (IFM), which learns tangent velocity fields using Pancharatnam-aligned conditional paths.
    \item We prove that Pancharatnam alignment yields a canonical endpoint representative, that IFM recovers the induced marginal transport field, and that horizontal intrinsic parameterization removes redundant ambient directions.
    \item We empirically compare IFM with Euclidean flow matching and intrinsic diffusion baselines across controlled single-cluster, multimodal, physics-inspired, and MNIST-derived amplitude-vector benchmarks, including dense $10$- and selected $12$-qubit tests plus an appendix $14$-qubit stress test.
\end{itemize}
\section{Related Work}

\paragraph{Intrinsic diffusion for quantum pure-state ensembles.}
Diffusion and score-based models are central in modern generative modeling \citep{ho2020denoising,song2021scorebased,song2022denoisingdiffusionimplicitmodels}, including recent quantum variants \citep{zhang2024generative,chen2024quantum,kolle2024quantum,kwun2025mixed}.
The closest precursor is SSDM, which formulates score-based generation intrinsically on the quantum pure-state manifold using stochastic Schr\"odinger-type dynamics, Fubini--Study geometry, local teacher constructions, and reverse-time sampling \citep{Gisin1992TheQD,bouten2004stochastic,xu2026ssdm}.
IFM keeps the same projective state space but replaces score learning and stochastic reversal with deterministic velocity matching.

\paragraph{Flow matching on manifolds and transport geometries.}
Flow matching replaces score estimation and reverse stochastic simulation with direct regression to path-induced velocity fields \citep{lipman2023flowmatchinggenerativemodeling,lipman2022flow,tong2023conditional}.
Rectified, optimal, and Riemannian variants further emphasize the roles of path geometry, transport couplings, and manifold structure \citep{liu2022flow,liu2022rectified,lee2024improving,mathieu2020riemannian,huang2022riemannian,bortoli2022riemannian,lou2023scaling,tang2024adaptivity,falorsi2021continuous,miller2024flowmm,klein2023equivariant,gat2024discrete,chen2023flow}.
IFM specializes this logic to quantum pure states by using Pancharatnam-aligned endpoints, horizontal tangent fields, and projective transport on $\mathbb{CP}^{d-1}$.

\paragraph{Alternative quantum representations for flow-based generation.}
\emph{Quantum Flow Matching} (QFM) studies circuit-based density-matrix interpolation \citep{cui2025quantum}, while related flow generators operate on spin Wigner functions, functions, or classical-shadow representations \citep{hahm2024generative,kerrigan2023functional,hahmshadowfm}.
These works are conceptually complementary but use different state representations: IFM transports pure-state ensembles directly on $\mathbb{CP}^{d-1}$.
Appendix~\ref{app:additional_related_work} extends this discussion to optimal transport on structured spaces, geometric phase and Pancharatnam alignment, Riemannian flow matching on general geometries, and physics-inspired transport viewpoints.

\section{Preliminaries}

\paragraph{Quantum pure-state distributions on $\mathbb{CP}^{d-1}$.}
An $n$-qubit pure state is a unit vector $\psi \in \mathbb{C}^{d}$, $d=2^n$, modulo global phase, so the state space is \citep{mielnik1968geometry,bengtsson2017geometry,ogiue1969complex,wu2022metric}
\begin{equation}
    \mathbb{CP} ^{d-1} = \{ \psi \in \mathbb{C}^d : \|\psi\|_2 = 1 \} / U(1).
\end{equation}
We seek a generative map from a base density $p_0$ to a target ensemble $p_1$ on this projective manifold, so both density evolution and velocity fields should be intrinsic rather than tied to arbitrary ambient representatives.

\paragraph{Stochastic diffusion versus deterministic transport.}
Intrinsic diffusion models on $\mathbb{CP}^{d-1}$ lead to score-based objectives and reverse-time stochastic simulation \citep{xu2026ssdm}.
We instead model a density $p_t$ moving under a tangent velocity field $v_t(\psi)\in T_\psi\mathbb{CP}^{d-1}$, governed by the manifold continuity equation \citep{delgado1999quantum,wyatt2005quantum,chewi2021fast}
\begin{equation}
    \partial_t p_t + \mathrm{div}_{\mathbb{CP} }(p_t v_t) = 0,
\end{equation}
where $\mathrm{div}_{\mathbb{CP}}$ is the intrinsic divergence.
Generation is then realized by learning a velocity field that transports $p_0$ toward $p_1$ along manifold-consistent trajectories.

\paragraph{Flow matching.}
Flow matching learns a time-dependent vector field by regressing to conditional velocities induced by a prescribed path between base and data samples \citep{lipman2023flowmatchinggenerativemodeling,lipman2022flow,liu2022flow,tong2023conditional}.
Sampling then integrates the learned ordinary differential equation
\begin{equation}
    \frac{d x_t}{dt} = v_\theta(x_t,t).
\end{equation}
We lift this principle from Euclidean space to quantum pure-state manifolds, following recent manifold and Riemannian flow-matching developments \citep{chen2023flow,miller2024flowmm,klein2023equivariant}.

\section{Intrinsic Flow Matching on Quantum Pure-State Manifolds}

We introduce \emph{Intrinsic Flow Matching} (IFM), which models quantum pure-state ensembles as deterministic probability transport on $\mathbb{CP}^{d-1}$ \citep{mielnik1968geometry,bengtsson2017geometry,lipman2023flowmatchinggenerativemodeling,chen2023flow}.
Rather than using stochastic diffusion and reverse-time score estimation, IFM learns a time-dependent tangent velocity field that transports a base density $p_0$ to a target ensemble density $p_1$ along geometry-consistent projective trajectories.

\begin{figure*}[t]
    \centering
    \includegraphics[width=0.96\textwidth]{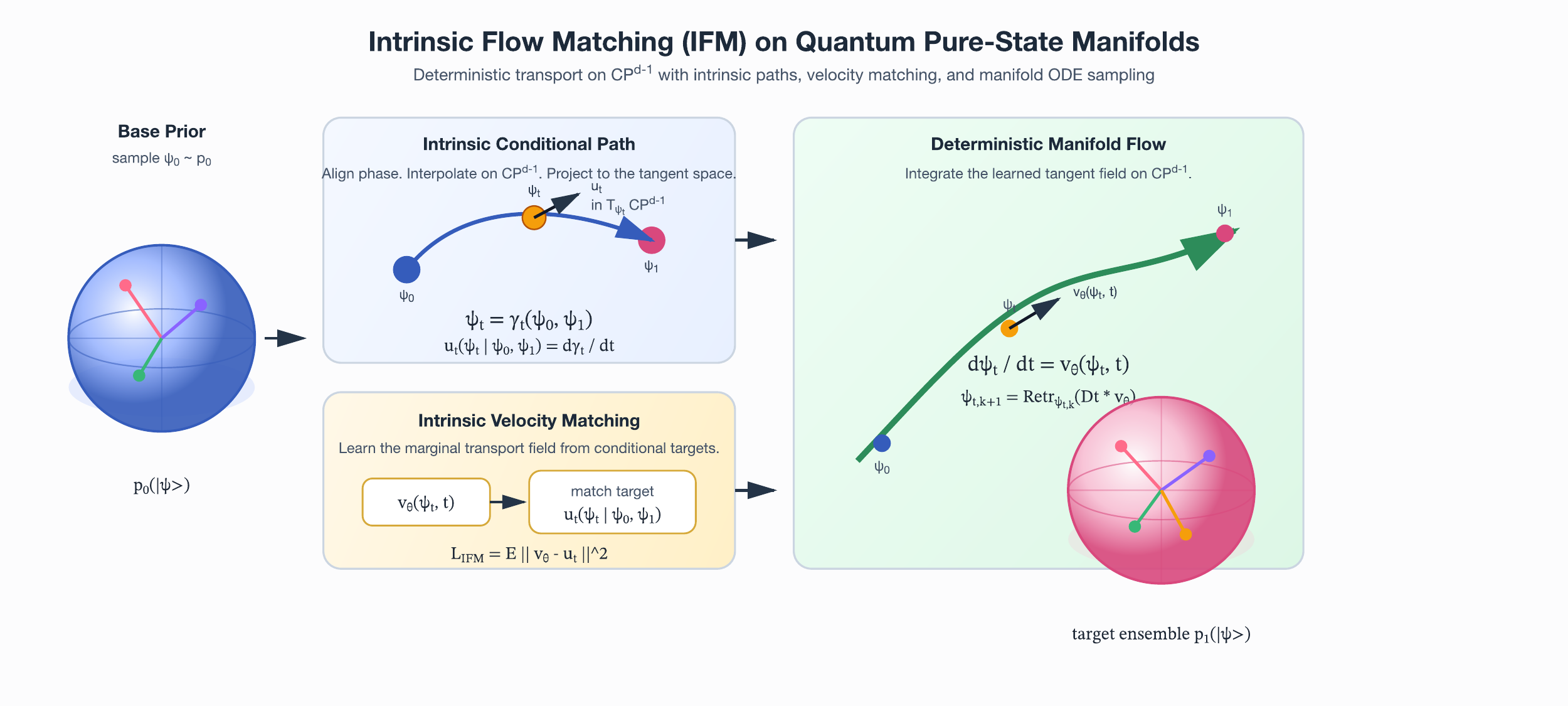}
    \caption{IFM on $\mathbb{CP}^{d-1}$: phase-align endpoints, regress tangent velocities, and sample by deterministic manifold flow.}
    \label{fig:ifm_overview}
\end{figure*}

\subsection{Intrinsic transport formulation}

Let $\psi_t \in \mathbb{CP} ^{d-1}$ denote a pure-state trajectory and let
\begin{equation}
    v_t(\psi) \in T_\psi \mathbb{CP} ^{d-1}
\end{equation}
be a tangent velocity field.
The induced density evolution is governed by the manifold continuity equation \citep{chewi2021fast,delgado1999quantum,wyatt2005quantum}
\begin{equation}
    \partial_t p_t + \mathrm{div}_{\mathbb{CP} }(p_t v_t) = 0.
\end{equation}
We learn a neural approximation $v_\theta(\psi,t)$ so that integrating the flow from $p_0$ produces samples from $p_1$.
Using unit-sphere representatives and quotienting out global phase, we identify the projective tangent space with the horizontal subspace
\begin{equation}
    \mathcal{H}_\psi
    :=
    \left\{
        \xi \in \mathbb{C}^d : \langle \psi, \xi \rangle = 0
    \right\}.
\end{equation}
This removes both radial motion and unphysical phase rotation. Under the Fubini--Study metric, the horizontal-coordinate inner product is \citep{mielnik1968geometry,bengtsson2017geometry}
\begin{equation}
    g_{\mathrm{FS},\psi}(\xi,\eta)
    :=
    \mathrm{Re}\,\langle \xi,\eta\rangle,
    \qquad
    \|\xi\|_{\psi}^2 := g_{\mathrm{FS},\psi}(\xi,\xi).
\end{equation}
All IFM vector fields, losses, and numerical updates are defined in this tangent-space representation.
Compared with diffusion-based quantum generative models \citep{ho2020denoising,song2021scorebased,xu2026ssdm}, this replaces score learning, local diffusion teachers, and reverse-time stochastic simulation with direct velocity matching and deterministic manifold flow integration.

\subsection{Conditional manifold paths}
\label{sec:conditional-paths}

Flow matching requires conditional paths between base and target samples \citep{lipman2023flowmatchinggenerativemodeling,lipman2022flow,tong2023conditional}.
For $\psi_0\sim p_0$ and $\psi_1\sim p_1$, we define a path on projective pure states,
\begin{equation}
    \psi_t = \gamma_t(\psi_0,\psi_1), \qquad t \in [0,1],
\end{equation}
with endpoints $\gamma_0=\psi_0$ and $\gamma_1=\psi_1$, and regress to its tangent velocity $u_t=\frac{d}{dt}\gamma_t$.
Because $[\psi]=[e^{i\phi}\psi]$, the endpoint representatives must be gauge-fixed before interpolation.
We use the Pancharatnam in-phase convention, choosing the representative of $\psi_1$ whose overlap with $\psi_0$ is real and nonnegative \citep{pancharatnam1956generalized}.

For pairs with $\langle \psi_0,\psi_1\rangle \neq 0$, define
\begin{equation}
    \alpha(\psi_0,\psi_1)
    :=
    \frac{\langle \psi_0,\psi_1\rangle}{|\langle \psi_0,\psi_1\rangle|},
    \qquad
    \bar{\psi}_1 := \alpha(\psi_0,\psi_1)^{-1}\psi_1.
\end{equation}
Proposition~\ref{prop:pancharatnam_alignment} (Appendix~\ref{app:phase_alignment}) proves that this is the Pancharatnam-aligned representative: it maximizes the real endpoint overlap (Lemma~\ref{lem:phase_alignment_real_overlap}) and minimizes the ambient chord length (Lemma~\ref{lem:phase_alignment_shortest_chord}) among phase-equivalent representatives.

When $\langle \psi_0,\psi_1\rangle = 0$, any unit-modulus phase gives an equivalent representative; in practice we set $\bar{\psi}_1=\psi_1$.

We then use normalized chord interpolation,
\begin{equation}
    \tilde{\psi}_t := (1-t)\psi_0 + t\bar{\psi}_1,
\end{equation}
\begin{equation}
    \psi_t = \gamma_t(\psi_0,\psi_1)
    :=
    \frac{\tilde{\psi}_t}{\|\tilde{\psi}_t\|_2}.
\end{equation}
This gives a closed-form, inexpensive projective path that avoids spurious phase motion.
Appendix~\ref{app:phase_alignment} develops this Pancharatnam-alignment viewpoint in more detail, and Appendix~\ref{app:interpolation_details} states the full interpolation rules used in our experiments, including the geodesic IFM variant and the generic manifold-geodesic baseline.

Differentiating the normalized path yields an explicit tangent velocity.
Let
\begin{equation}
    \dot{\tilde{\psi}}_t = \bar{\psi}_1 - \psi_0.
\end{equation}
Using $\psi_t = \tilde{\psi}_t / \|\tilde{\psi}_t\|_2$, we obtain
\begin{equation}
    u_t(\psi_t \mid \psi_0,\psi_1)
    :=
    \dot{\psi}_t
    =
    \frac{1}{\|\tilde{\psi}_t\|_2}
    \Pi_{\psi_t}\!\left(\bar{\psi}_1 - \psi_0\right),
\end{equation}
where $\Pi_{\psi_t}$ denotes the orthogonal projection onto the horizontal tangent space at $\psi_t$.
Hence $u_t(\psi_t \mid \psi_0,\psi_1) \in \mathcal{H}_{\psi_t}$ by construction.
Thus IFM trains a neural tangent field against a phase-aligned, horizontally projected conditional transport direction.
Proposition~\ref{prop:path_welldefined} (Appendix~\ref{app:proof_path_welldefined}) shows that the resulting path is well-defined on projective classes and that its velocity always lies in the horizontal tangent bundle.

\subsection{Intrinsic flow matching objective}

Let $v_\theta(\psi,t)\in T_\psi\mathbb{CP}^{d-1}$ be a tangent neural velocity field, implemented by projecting the network output onto the tangent space.
For $(\psi_0,\psi_1)$ and $t\sim\mathrm{Unif}[0,1]$, set $\psi_t=\gamma_t(\psi_0,\psi_1)$ and minimize
\begin{equation}
    \mathcal{L}_{\mathrm{IFM}}(\theta)
    =
    \mathbb{E}_{\psi_0 \sim p_0,\;\psi_1 \sim p_1,\;t \sim \mathrm{Unif}[0,1]}
    \left[
        \left\|
            v_\theta(\psi_t,t) - u_t(\psi_t \mid \psi_0,\psi_1)
        \right\|^2_{\psi_t}
    \right],
\end{equation}
where $\|\cdot\|_{\psi_t}$ denotes the norm induced by the manifold metric at $\psi_t$.

This is Euclidean flow matching lifted to $\mathbb{CP}^{d-1}$ \citep{lipman2023flowmatchinggenerativemodeling,tong2023conditional}: both the target and learned velocities are tangent, so the learned transport is intrinsic rather than ambient.

\subsection{Population characterization of the IFM objective}
\label{sec:population-characterization}

We now characterize the vector field recovered by learning the phase-aligned target in expectation.
The results identify the population minimizer, connect it to marginal continuity dynamics and endpoint transport, and give a Wasserstein stability bound for approximate minimizers.
Let $(\psi_0,\psi_1)$ be sampled from a joint distribution $\pi$ whose marginals are $p_0$ and $p_1$, and define the intermediate random variable
\begin{equation}
    \psi_t = \gamma_t(\psi_0,\psi_1), \qquad
    p_t := \text{distribution of } \psi_t.
\end{equation}
The conditional velocity $u_t(\psi_t \mid \psi_0,\psi_1)$ induces a probability current along this path family, paralleling current-based descriptions of quantum density evolution \citep{delgado1999quantum,hodge2014electron,wyatt2005quantum}.
The key question is whether minimizing $\mathcal{L}_{\mathrm{IFM}}$ recovers the correct marginal transport field on $\mathbb{CP}^{d-1}$.

\begin{theorem}[Population optimality of intrinsic flow matching]
\label{thm:ifm_population_opt}
Fix $t \in (0,1)$ and consider the population objective over measurable tangent vector fields $v(\cdot,t): \mathbb{CP}^{d-1} \to T\mathbb{CP}^{d-1}$,
\begin{equation}
    \mathcal{J}_t[v]
    :=
    \mathbb{E}\left[
        \|v(\psi_t,t) - u_t(\psi_t \mid \psi_0,\psi_1)\|_{\psi_t}^2
    \right].
\end{equation}
Then any minimizer $v_t^\star$ of $\mathcal{J}_t$ satisfies, for $p_t$-almost every $\psi$,
\begin{equation}
    v_t^\star(\psi)
    =
    \mathbb{E}\big[
        u_t(\psi_t \mid \psi_0,\psi_1)\,\big|\, \psi_t=\psi
    \big].
    \label{eq:ifm_conditional_expectation}
\end{equation}
In particular, the population IFM objective learns the conditional-expectation velocity field associated with the chosen manifold path family.
\end{theorem}

\noindent\textit{Proof.} See Appendix~\ref{app:proof_population_opt}.

\begin{remark}
Theorem~\ref{thm:ifm_population_opt} plays for IFM the same role that score-recovery results play in diffusion models \citep{song2021scorebased,ho2020denoising}: it identifies the exact population object learned by the training objective.
\end{remark}

\begin{proposition}[Conditional-to-marginal IFM objective decomposition]
\label{prop:ifm_loss_decomposition}
Let
\begin{equation}
    \bar{u}_t(\psi)
    :=
    \mathbb{E}\big[
        u_t(\psi_t \mid \psi_0,\psi_1)\,\big|\, \psi_t=\psi
    \big]
\end{equation}
denote the marginal velocity field induced by the conditional path family.
Then for every measurable tangent vector field $v(\cdot,t)$,
\begin{equation}
    \mathcal{J}_t[v]
    =
    \mathbb{E}\big[
        \|v(\psi_t,t)-\bar{u}_t(\psi_t)\|_{\psi_t}^2
    \big]
    +
    \mathbb{E}\big[
        \|u_t(\psi_t \mid \psi_0,\psi_1)-\bar{u}_t(\psi_t)\|_{\psi_t}^2
    \big].
\end{equation}
In particular, the conditional IFM objective and the marginal regression objective differ by an additive constant independent of $v$, and therefore have the same minimizers.
\end{proposition}

\noindent\textit{Proof.} See Appendix~\ref{app:proof_loss_decomposition}.

\begin{proposition}[Continuity equation induced by the population IFM field]
\label{prop:ifm_continuity}
Assume $t \mapsto \gamma_t(\psi_0,\psi_1)$ is $C^1$ and that the induced marginal law $p_t$ admits a smooth density with respect to the Fubini--Study volume measure.
Let $v_t^\star$ be defined by \eqref{eq:ifm_conditional_expectation}.
Then $p_t$ satisfies the intrinsic continuity equation
\begin{equation}
    \partial_t p_t + \mathrm{div}_{\mathbb{CP}}(p_t v_t^\star) = 0
\end{equation}
in the weak sense on $\mathbb{CP}^{d-1}$.
\end{proposition}

\noindent\textit{Proof.} See Appendix~\ref{app:proof_continuity}.

\begin{remark}
Proposition~\ref{prop:ifm_continuity} formalizes the transport interpretation emphasized in the abstract and introduction: the field learned by IFM is not merely a regression target, but the velocity field governing density evolution on $\mathbb{CP}^{d-1}$.
\end{remark}

\begin{theorem}[Endpoint-preserving transport]
\label{thm:ifm_endpoint}
Suppose the conditional path satisfies $\gamma_0(\psi_0,\psi_1)=\psi_0$ and $\gamma_1(\psi_0,\psi_1)=\psi_1$ almost surely, and that the population minimizer $v_t^\star$ from Theorem~\ref{thm:ifm_population_opt} is locally Lipschitz on $\mathbb{CP}^{d-1}$ uniformly in $t \in [0,1]$.
Then the manifold ODE
\begin{equation}
    \frac{d}{dt}\psi_t = v_t^\star(\psi_t),\qquad \psi_0 \sim p_0,
\end{equation}
admits a unique flow $\Phi_{0,t}: \mathbb{CP}^{d-1}\to\mathbb{CP}^{d-1}$, and for every $t\in[0,1]$,
\begin{equation}
    (\Phi_{0,t})_\# p_0 = p_t.
\end{equation}
In particular, $(\Phi_{0,1})_\# p_0 = p_1$.
\end{theorem}

\noindent\textit{Proof.} See Appendix~\ref{app:proof_endpoint}.

\begin{remark}
Thus, as an immediate corollary, the population IFM field gives sample-level transport: its deterministic ODE pushes $p_0$ exactly onto $p_1$, not merely onto matching intermediate marginals.
The argument combines Proposition~\ref{prop:ifm_continuity} with the standard uniqueness of weak solutions to the manifold continuity equation under Lipschitz transport \citep{ambrosio2008transport,villani2008stability}; compactness of $\mathbb{CP}^{d-1}$ removes the usual growth-condition technicalities of the Euclidean case.
\end{remark}

\begin{proposition}[$L^2$-to-Wasserstein stability]
\label{prop:ifm_stability}
Let $\hat{v}: \mathbb{CP}^{d-1}\times[0,1]\to T\mathbb{CP}^{d-1}$ be a tangent vector field with flow $\hat{\Phi}_{0,t}$, and let $\hat p_t := (\hat\Phi_{0,t})_\# p_0$.
Assume both $v_t^\star$ and $\hat v_t$ are $L$-Lipschitz on $\mathbb{CP}^{d-1}$ in the unit-sphere representation, uniformly in $t \in [0,1]$, and that
\begin{equation}
    \int_0^1 \mathbb{E}_{\psi\sim p_t}\big\|\hat{v}(\psi,t)-v_t^\star(\psi)\big\|^2_\psi\, dt \;\le\; \varepsilon^2.
\end{equation}
Then the Wasserstein-$2$ distance with respect to the Fubini--Study metric satisfies
\begin{equation}
    W_2^{\mathrm{FS}}\!\big(\hat p_1,\, p_1\big) \;\le\; C_L\,\varepsilon,
\end{equation}
with a constant $C_L = e^{L}$ depending only on $L$.
\end{proposition}

\noindent\textit{Proof.} See Appendix~\ref{app:proof_stability}.

\begin{remark}
Proposition~\ref{prop:ifm_stability} is a quantitative bridge between training error and sample quality on the projective manifold: an $\varepsilon$-approximate IFM minimizer in the population $L^2$ sense yields a generated distribution $\hat p_1$ that is $O(\varepsilon)$ close to $p_1$ in Fubini--Study Wasserstein distance, with a constant that depends only on the common Lipschitz constant $L$ of the two velocity fields.
The bound is the manifold counterpart of standard Gr\"onwall stability for ODE flows \citep{villani2008stability,ambrosio2008transport,lee2012smooth} and does not assume $\hat v$ is itself a minimizer of the IFM loss.
The assumption that the trial field $\hat v$ is also Lipschitz is mild: any neural network with bounded weights and Lipschitz nonlinearities satisfies it, and we may take $L$ to be the larger of the two constants.
\end{remark}

\begin{proposition}[Projective flow compatibility]
\label{prop:physical_compatibility}
Let $X_t(\psi) \in T_\psi \mathbb{CP}^{d-1}$ be a $C^1$ tangent field with flow $\Phi_{s,t}$ generated by
\begin{equation}
    \frac{d}{dt}\psi_t = X_t(\psi_t).
\end{equation}
For $p_t=(\Phi_{0,t})_\#p_0$, the family $(p_t)$ satisfies
\begin{equation}
    \partial_t p_t + \mathrm{div}_{\mathbb{CP}}(p_t X_t) = 0
\end{equation}
weakly. If the conditional paths coincide with these trajectories,
\(
\gamma_t(\psi_0,\psi_1)=\Phi_{0,t}(\psi_0)
\)
and
\(
\psi_1=\Phi_{0,1}(\psi_0)
\),
then the population IFM minimizer obeys $v_t^\star(\psi)=X_t(\psi)$ for $p_t$-a.e.\ $\psi$.
\end{proposition}

\noindent\textit{Proof.} See Appendix~\ref{app:proof_physical_compatibility}.

\begin{remark}
Proposition~\ref{prop:physical_compatibility} connects IFM to physics-inspired transport without identifying it with a specific Bohmian, hydrodynamic, or Hamiltonian model: any deterministic tangent ensemble flow lies in IFM's population-level representational scope when the path family matches that flow.
\end{remark}

\begin{proposition}[Monge-geodesic specialization]
\label{prop:ifm_monge_geodesic}
Assume the coupling $\pi$ is deterministic, i.e.,
\begin{equation}
    \pi = (\mathrm{Id},T)_\# p_0
\end{equation}
for a measurable map $T : \mathbb{CP}^{d-1} \to \mathbb{CP}^{d-1}$; equivalently, after sampling $\psi_0 \sim p_0$ we set $\psi_1=T(\psi_0)$ almost surely.
Assume moreover that the conditional path is chosen as the Fubini--Study geodesic interpolation
\begin{equation}
    \psi_t = \gamma_t(\psi_0,T(\psi_0))
\end{equation}
and that, for each $t\in(0,1)$, the map $\psi_0 \mapsto \gamma_t(\psi_0,T(\psi_0))$ is injective on a full-measure subset of $p_0$.
Then the population IFM minimizer satisfies
\begin{equation}
    v_t^\star(\psi_t)
    =
    \frac{d}{dt}\gamma_t(\psi_0,T(\psi_0))
\end{equation}
for $p_t$-almost every $\psi_t$.
Consequently, IFM exactly recovers the deterministic velocity field of the prescribed geodesic interpolation.
If, in addition, $T$ is the optimal Monge map for the quadratic Fubini--Study cost, then $v_t^\star$ coincides with the velocity field of the corresponding displacement interpolation.
\end{proposition}

\noindent\textit{Proof.} See Appendix~\ref{app:proof_monge_geodesic}.

\paragraph{Scope of the theory.}
The population-identification and continuity-equation pieces, Theorem~\ref{thm:ifm_population_opt} and Propositions~\ref{prop:ifm_loss_decomposition}--\ref{prop:ifm_continuity}, are direct $\mathbb{CP}^{d-1}$ analogues of standard flow-matching results \citep{lipman2023flowmatchinggenerativemodeling,tong2023conditional,lipman2022flow}; the geometry-specific parts are the Pancharatnam path, horizontal minimal-norm parameterization, endpoint transport, Fubini--Study stability, and Monge--geodesic specialization. Neural parameterization, projection, and deterministic retraction sampling are deferred to Appendices~\ref{app:parameterization} and~\ref{sec:sampling}.

\section{Experiments}
\label{sec:experiments}

We evaluate whether phase-aligned intrinsic transport on $\mathbb{CP}^{d-1}$ improves over ambient Euclidean transport for quantum pure-state generation.
The main baseline is \emph{Euclidean flow matching} (Euclidean FM), which uses the same flow-matching objective but learns in an ambient real state-vector representation followed by normalization.

\subsection{Experimental setup}

All experiments use dense pure-state ensembles on $n$ qubits, with Hilbert-space dimension $d=2^n$ \citep{bengtsson2017geometry,preskill2018nisq}. We evaluate controlled single-cluster scaling, coherence-sensitive multimodal ensembles (whose Bloch-sphere geometry is illustrated in Appendix Figure~\ref{fig:benchmark_bloch_schematics}), spin-coherent peaks, compact TFIM/XXZ ground-state families, and MNIST-derived amplitude-encoded image-vector states, with an additional $14$-qubit stress test (Appendix~\ref{app:14q_flow_stress}, Table~\ref{tab:appendix_14q_flow_stress}, Figure~\ref{fig:longrun_14q_cluster}). These benchmarks are chosen to expose geometric effects in a controlled way. Unless otherwise stated, IFM and Euclidean FM are trained with matched network scale, optimization budget, sample budget, and evaluation protocol, and main tables report mean $\pm$ standard deviation over $10$ seeds; Appendix~\ref{app:exp_overview} gives the full benchmark construction (Appendix~\ref{app:benchmark_construction}, Table~\ref{tab:appendix_benchmark_summary}), training settings, network architectures (Appendix~\ref{app:model_architectures}), baseline configurations (Appendix~\ref{app:baseline_configs}, Table~\ref{tab:appendix_baseline_configs}), sampling and reporting details (Appendix~\ref{app:reporting_details}), and code release information (Appendix~\ref{app:code_contribution}). A summary heatmap of IFM-to-Euclidean-FM error ratios across all benchmarks is given in Appendix Figure~\ref{fig:experiment_ratio_heatmap}; further methodological discussion accompanying the IFM algorithm and its sampler is collected in Appendix~\ref{app:method_discussion_algorithm}, including a scalability discussion (Appendix~\ref{sec:scalability}), a comparison with SSDM and Euclidean flow matching (Appendix~\ref{sec:relation}), and the training-and-sampling pseudocode (Appendix~\ref{sec:algorithm}, Algorithm~\ref{alg:ifm}).

\subsection{Baseline and metrics}

Our main baseline is Euclidean FM \citep{lipman2023flowmatchinggenerativemodeling,lipman2022flow}, which uses the same flow-matching principle in an ambient state-vector representation. We also compare with SSDM \citep{xu2026ssdm}, Euclidean VP-SDE \citep{song2021scorebased}, QDDPM \citep{zhang2024generative}, QGAN \citep{lloyd2018quantum,dallaire2018quantum}, and intrinsic geodesic-flow ablations where relevant. We report MMD \citep{gretton2012mmd}, observable discrepancy $\Delta_{\mathrm{obs}}$, and entanglement Wasserstein distance; all definitions and baseline-specific settings are in Appendix~\ref{app:exp_protocol}--\ref{app:baseline_configs}.

\begin{table}[t]
\centering
\scriptsize
\caption{Representative $6$-qubit single-cluster comparison across diffusion, circuit, and flow baselines. Lower is better. Entries are mean $\pm$ standard deviation across $10$ random seeds (Appendix~\ref{app:exp_protocol}). The geodesic IFM row is an ablation of our method.}
\label{tab:single_cluster_baselines}
\setlength{\tabcolsep}{4pt}
\begin{tabular}{lccc}
\toprule
Method & MMD $\downarrow$ & $\Delta_{\mathrm{obs}} \downarrow$ & Ent.~$W_1 \downarrow$ \\
\midrule
Euclidean VP-SDE~\citep{song2021scorebased} & $(5.12 \pm 0.41) \times 10^{-1}$ & $(6.34 \pm 0.38) \times 10^{-1}$ & $(6.70 \pm 0.42) \times 10^{-1}$ \\
SSDM~\citep{xu2026ssdm} & $(1.73 \pm 0.14) \times 10^{-1}$ & $(4.16 \pm 0.27) \times 10^{-1}$ & $(5.59 \pm 0.31) \times 10^{-1}$ \\
QDDPM~\citep{zhang2024generative} & $(4.78 \pm 0.36) \times 10^{-1}$ & $(6.71 \pm 0.39) \times 10^{-1}$ & $(6.51 \pm 0.41) \times 10^{-1}$ \\
QGAN~\citep{lloyd2018quantum,dallaire2018quantum} & $(4.76 \pm 0.41) \times 10^{-1}$ & $(6.76 \pm 0.45) \times 10^{-1}$ & $(6.83 \pm 0.46) \times 10^{-1}$ \\
Euclidean FM~\citep{lipman2023flowmatchinggenerativemodeling} & $(6.02 \pm 0.51) \times 10^{-3}$ & $(2.16 \pm 0.18) \times 10^{-3}$ & $(8.15 \pm 0.62) \times 10^{-3}$ \\
Geodesic IFM (ours variant) & $(2.49 \pm 0.21) \times 10^{-4}$ & $(2.48 \pm 0.19) \times 10^{-3}$ & $(4.34 \pm 0.31) \times 10^{-3}$ \\
IFM (ours) & $(2.63 \pm 0.18) \times 10^{-4}$ & $(2.34 \pm 0.17) \times 10^{-3}$ & $(4.66 \pm 0.29) \times 10^{-3}$ \\
\bottomrule
\end{tabular}
\end{table}

\subsection{Single-cluster scaling}

\begin{table*}[t]
\centering
\caption{Single-cluster scaling against Euclidean FM \citep{lipman2023flowmatchinggenerativemodeling}. Lower is better; ``$\approx 0$'' denotes output-precision rounding. Entries are mean $\pm$ standard deviation across $10$ random seeds (Appendix~\ref{app:exp_protocol}).}
\label{tab:cluster_scaling}
\begin{tabular}{llccc}
\toprule
Qubits & Method & MMD $\downarrow$ & $\Delta_{\mathrm{obs}} \downarrow$ & Ent.~$W_1 \downarrow$ \\
\midrule
2 & Euclidean FM & $(2.45 \pm 0.21) \times 10^{-4}$ & $(5.94 \pm 0.42) \times 10^{-3}$ & $(5.45 \pm 0.38) \times 10^{-3}$ \\
2 & IFM (ours) & $\approx 0$ & $(1.13 \pm 0.09) \times 10^{-3}$ & $(1.67 \pm 0.12) \times 10^{-3}$ \\
\addlinespace
6 & Euclidean FM & $(6.02 \pm 0.51) \times 10^{-3}$ & $(2.16 \pm 0.18) \times 10^{-3}$ & $(8.15 \pm 0.62) \times 10^{-3}$ \\
6 & IFM (ours) & $(2.63 \pm 0.18) \times 10^{-4}$ & $(2.34 \pm 0.17) \times 10^{-3}$ & $(4.66 \pm 0.29) \times 10^{-3}$ \\
\addlinespace
10 & Euclidean FM & $(1.68 \pm 0.13) \times 10^{-2}$ & $(9.51 \pm 0.71) \times 10^{-2}$ & $(8.21 \pm 0.55) \times 10^{-2}$ \\
10 & IFM (ours) & $(7.38 \pm 0.51) \times 10^{-4}$ & $(1.39 \pm 0.12) \times 10^{-3}$ & $(1.46 \pm 0.11) \times 10^{-3}$ \\
\addlinespace
12 & Euclidean FM & $(8.76 \pm 0.65) \times 10^{-3}$ & $(3.08 \pm 0.22) \times 10^{-2}$ & $(1.50 \pm 0.10) \times 10^{-2}$ \\
12 & IFM (ours) & $(6.54 \pm 0.45) \times 10^{-4}$ & $(1.81 \pm 0.14) \times 10^{-3}$ & $(1.15 \pm 0.09) \times 10^{-3}$ \\
\bottomrule
\end{tabular}
\end{table*}

Table~\ref{tab:cluster_scaling} shows that IFM improves over Euclidean FM across the single-cluster scaling sweep, with the largest gains appearing at $10$ and $12$ qubits.
The representative $6$-qubit comparison in Table~\ref{tab:single_cluster_baselines} further shows that IFM is substantially stronger than SSDM, Euclidean VP-SDE, QDDPM, and QGAN on this benchmark.
To rule out undertraining as a source of the gap, we ran a $20{,}000$-step comparison on $14$-qubit single-cluster (Figure~\ref{fig:longrun_14q_cluster}); the Euclidean FM gap persists.
Appendix~\ref{app:single_cluster_discussion} gives additional scaling interpretation, the $10$-qubit scalable-baseline extension (Table~\ref{tab:appendix_single_cluster_10q}), and training-time details (Table~\ref{tab:appendix_baseline_timing}).

\subsection{Structured multimodal and coherence-sensitive ensembles}

\begin{table*}[t]
\centering
\scriptsize
\caption{Structured, physics-inspired, and amplitude-encoded MNIST benchmarks. Lower is better. Baselines are Euclidean FM \citep{lipman2023flowmatchinggenerativemodeling} and SSDM \citep{xu2026ssdm}; $6$-qubit extensions are in Appendix~\ref{app:6q_structured}. Entries are mean $\pm$ standard deviation across $10$ random seeds (Appendix~\ref{app:exp_protocol}).}
\label{tab:structured_benchmarks}
\begin{tabular}{llccc}
\toprule
Benchmark & Method & MMD $\downarrow$ & $\Delta_{\mathrm{obs}} \downarrow$ & Ent.~$W_1 \downarrow$ \\
\midrule
Eq. bimodal ($2$ qubits) & Euclidean FM & $(4.06 \pm 0.31) \times 10^{-2}$ & $(1.58 \pm 0.11) \times 10^{-1}$ & $(2.16 \pm 0.15) \times 10^{-1}$ \\
Eq. bimodal ($2$ qubits) & SSDM & $(2.39 \pm 0.18) \times 10^{-2}$ & $(4.48 \pm 0.32) \times 10^{-2}$ & $(2.38 \pm 0.17) \times 10^{-2}$ \\
Eq. bimodal ($2$ qubits) & IFM (ours) & $(2.33 \pm 0.17) \times 10^{-2}$ & $(1.05 \pm 0.08) \times 10^{-1}$ & $(7.55 \pm 0.49) \times 10^{-2}$ \\
\addlinespace
Eq. bimodal ($10$ qubits) & Euclidean FM & $(1.22 \pm 0.09) \times 10^{-2}$ & $(4.60 \pm 0.32) \times 10^{-2}$ & $(4.60 \pm 0.31) \times 10^{-2}$ \\
Eq. bimodal ($10$ qubits) & SSDM & $(7.11 \pm 0.51) \times 10^{-3}$ & $(5.64 \pm 0.40) \times 10^{-2}$ & $(5.04 \pm 0.36) \times 10^{-2}$ \\
Eq. bimodal ($10$ qubits) & IFM (ours) & $(8.38 \pm 0.61) \times 10^{-4}$ & $(7.99 \pm 0.55) \times 10^{-3}$ & $(1.64 \pm 0.11) \times 10^{-2}$ \\
\addlinespace
Trimodal ($2$ qubits) & Euclidean FM & $(1.75 \pm 0.13) \times 10^{-2}$ & $(1.00 \pm 0.07) \times 10^{-1}$ & $(1.13 \pm 0.08) \times 10^{-1}$ \\
Trimodal ($2$ qubits) & SSDM & $(5.04 \pm 0.36) \times 10^{-3}$ & $(6.22 \pm 0.42) \times 10^{-2}$ & $(4.18 \pm 0.29) \times 10^{-2}$ \\
Trimodal ($2$ qubits) & IFM (ours) & $\approx 0$ & $(3.34 \pm 0.24) \times 10^{-2}$ & $(9.47 \pm 0.62) \times 10^{-2}$ \\
\addlinespace
Trimodal ($10$ qubits) & Euclidean FM & $(1.27 \pm 0.09) \times 10^{-2}$ & $(5.97 \pm 0.40) \times 10^{-2}$ & $(5.81 \pm 0.39) \times 10^{-2}$ \\
Trimodal ($10$ qubits) & SSDM & $(7.78 \pm 0.55) \times 10^{-3}$ & $(7.25 \pm 0.51) \times 10^{-2}$ & $(6.12 \pm 0.42) \times 10^{-2}$ \\
Trimodal ($10$ qubits) & IFM (ours) & $(7.08 \pm 0.49) \times 10^{-4}$ & $(5.27 \pm 0.36) \times 10^{-3}$ & $(3.95 \pm 0.27) \times 10^{-3}$ \\
\addlinespace
Spin-coherent ($2$ qubits) & Euclidean FM & $(4.30 \pm 0.31) \times 10^{-2}$ & $(1.12 \pm 0.08) \times 10^{-1}$ & $(8.24 \pm 0.55) \times 10^{-2}$ \\
Spin-coherent ($2$ qubits) & SSDM & $(9.30 \pm 0.65) \times 10^{-3}$ & $(4.32 \pm 0.30) \times 10^{-2}$ & $(3.58 \pm 0.25) \times 10^{-2}$ \\
Spin-coherent ($2$ qubits) & IFM (ours) & $(5.27 \pm 0.36) \times 10^{-3}$ & $(3.72 \pm 0.26) \times 10^{-2}$ & $(3.76 \pm 0.27) \times 10^{-2}$ \\
\addlinespace
Spin-coherent ($10$ qubits) & Euclidean FM & $(1.22 \pm 0.09) \times 10^{-2}$ & $(1.43 \pm 0.11) \times 10^{-3}$ & $(8.03 \pm 0.54) \times 10^{-2}$ \\
Spin-coherent ($10$ qubits) & SSDM & $(7.26 \pm 0.51) \times 10^{-3}$ & $(1.41 \pm 0.10) \times 10^{-3}$ & $(8.68 \pm 0.59) \times 10^{-2}$ \\
Spin-coherent ($10$ qubits) & IFM (ours) & $(3.81 \pm 0.27) \times 10^{-4}$ & $(1.79 \pm 0.13) \times 10^{-3}$ & $(4.44 \pm 0.31) \times 10^{-2}$ \\
\addlinespace
TFIM ($2$ qubits) & Euclidean FM & $(4.87 \pm 0.34) \times 10^{-4}$ & $(7.43 \pm 0.51) \times 10^{-3}$ & $(1.43 \pm 0.10) \times 10^{-2}$ \\
TFIM ($2$ qubits) & SSDM & $(9.41 \pm 0.66) \times 10^{-4}$ & $(5.58 \pm 0.39) \times 10^{-3}$ & $(4.20 \pm 0.29) \times 10^{-2}$ \\
TFIM ($2$ qubits) & IFM (ours) & $\approx 0$ & $(6.68 \pm 0.46) \times 10^{-3}$ & $(4.04 \pm 0.28) \times 10^{-2}$ \\
\addlinespace
TFIM ($10$ qubits) & Euclidean FM & $(1.23 \pm 0.09) \times 10^{-2}$ & $(3.86 \pm 0.27) \times 10^{-2}$ & $(6.63 \pm 0.45) \times 10^{-2}$ \\
TFIM ($10$ qubits) & SSDM & $(7.06 \pm 0.50) \times 10^{-3}$ & $(5.06 \pm 0.36) \times 10^{-2}$ & $(7.02 \pm 0.49) \times 10^{-2}$ \\
TFIM ($10$ qubits) & IFM (ours) & $(1.01 \pm 0.07) \times 10^{-3}$ & $(8.91 \pm 0.62) \times 10^{-3}$ & $(2.95 \pm 0.20) \times 10^{-2}$ \\
\addlinespace
XXZ ($2$ qubits) & Euclidean FM & $(3.94 \pm 0.28) \times 10^{-2}$ & $(1.39 \pm 0.10) \times 10^{-2}$ & $(6.43 \pm 0.45) \times 10^{-2}$ \\
XXZ ($2$ qubits) & SSDM & $(8.60 \pm 0.61) \times 10^{-3}$ & $(4.84 \pm 0.34) \times 10^{-2}$ & $(7.19 \pm 0.49) \times 10^{-2}$ \\
XXZ ($2$ qubits) & IFM (ours) & $\approx 0$ & $(1.08 \pm 0.08) \times 10^{-2}$ & $(9.39 \pm 0.62) \times 10^{-2}$ \\
\addlinespace
XXZ ($10$ qubits) & Euclidean FM & $(1.31 \pm 0.09) \times 10^{-2}$ & $(2.22 \pm 0.16) \times 10^{-2}$ & $(4.23 \pm 0.30) \times 10^{-2}$ \\
XXZ ($10$ qubits) & SSDM & $(8.24 \pm 0.58) \times 10^{-3}$ & $(2.68 \pm 0.19) \times 10^{-2}$ & $(4.55 \pm 0.32) \times 10^{-2}$ \\
XXZ ($10$ qubits) & IFM (ours) & $(3.52 \pm 0.24) \times 10^{-4}$ & $(3.92 \pm 0.27) \times 10^{-3}$ & $(2.50 \pm 0.18) \times 10^{-2}$ \\
\addlinespace
MNIST $0/1 \rightarrow 6$q & Euclidean FM & $(2.10 \pm 0.15) \times 10^{-2}$ & $(1.30 \pm 0.09) \times 10^{-1}$ & $(2.75 \pm 0.18) \times 10^{-1}$ \\
MNIST $0/1 \rightarrow 6$q & IFM (ours) & $(4.44 \pm 0.31) \times 10^{-3}$ & $(5.14 \pm 0.36) \times 10^{-2}$ & $(5.64 \pm 0.39) \times 10^{-2}$ \\
\addlinespace
MNIST $0/1 \rightarrow 12$q & Euclidean FM & $(1.51 \pm 0.10) \times 10^{-1}$ & $(2.01 \pm 0.14) \times 10^{-1}$ & $(1.70 \pm 0.11) \times 10^{0}$ \\
MNIST $0/1 \rightarrow 12$q & IFM (ours) & $(1.40 \pm 0.10) \times 10^{-2}$ & $(9.33 \pm 0.65) \times 10^{-2}$ & $(1.10 \pm 0.08) \times 10^{0}$ \\
\bottomrule
\end{tabular}
\end{table*}

Table~\ref{tab:structured_benchmarks} shows the clearest IFM gains on $10$-qubit multimodal, spin-coherent, and compact physics-inspired benchmarks.
The advantage is strongest when modes differ by projective phase structure or coherent product-state organization, while the metric trade-offs remain mixed on some smoother families.
Appendices~\ref{app:structured_discussion}--\ref{app:fmgg_baseline} give the detailed $6$-qubit extensions (Table~\ref{tab:appendix_6q_structured}), phase-alignment ablation (Appendix~\ref{app:phase_ablation_6q}, Table~\ref{tab:appendix_phase_ablation}), the geometry/path-choice ablation summary (Appendix~\ref{app:geometry_path_ablation_6q}, Table~\ref{tab:appendix_geometry_path_summary}), and the generic manifold-geodesic FM comparison (Table~\ref{tab:appendix_fmgg_baseline}); that baseline is competitive on several $6$-qubit tasks and wins on XXZ, so our claim is not that IFM uniformly dominates every intrinsic path, but that quotient-aware phase alignment often helps on coherent projective structure.

\subsection{Physics-inspired families and amplitude-encoded image benchmarks}

The TFIM and XXZ families provide compact physics-inspired tests of transport over spin-chain ground-state families.
At $2$ qubits, IFM improves MMD and slightly improves observable discrepancy on both families, while Euclidean FM remains competitive on Ent.~$W_1$.
Within these compact families, the $10$-qubit results are more favorable to IFM: it improves all three metrics on both TFIM and XXZ relative to Euclidean FM and SSDM.
The MNIST-derived benchmark gives a complementary classical image-vector stress test after amplitude encoding; at both $6$ and $12$ qubits, this embedded image distribution favors IFM over Euclidean FM across all reported metrics.
Together with the structured multimodal benchmarks and the appendix $14$-qubit stress test, these results indicate that IFM is most useful in controlled settings where the distribution is high-dimensional, multimodal, coherence-sensitive, or naturally posed on projective state space.

\section{Conclusion}

We introduced \emph{Intrinsic Flow Matching} (IFM), a transport-based generative framework for quantum pure-state ensembles on $\mathbb{CP}^{d-1}$.
IFM replaces score learning and reverse-time stochastic simulation with learned tangent velocity fields and deterministic manifold sampling, while preserving the quotient geometry of pure states.
The theory shows that the resulting objective learns an intrinsic transport field for phase-aligned projective paths, while the Wasserstein stability bound links population velocity error to generated distribution quality.
The controlled experiments show frequent but not uniform gains over Euclidean flow baselines, with especially clear improvements on synthetic multimodal, coherence-sensitive, and higher-qubit benchmarks.
Overall, these results suggest that geometric transport is a promising organizing principle for quantum generative modeling.

\bibliographystyle{unsrt}
\bibliography{references}


\appendix

\section{Additional Related Work}
\label{app:additional_related_work}

\paragraph{Quantum generative modeling and ambient baselines.}
Generative modeling for quantum data has attracted increasing attention in recent years, with approaches ranging from quantum circuit Born machines and variational quantum generators to classical deep generative models for quantum states and observables \citep{lloyd2018quantum,dallaire2018quantum,liu2018differentiable,coyle2020born,benedetti2019generative,amin2018quantum}.
A common alternative is to embed quantum states into ambient real or complex vector spaces and apply standard Euclidean generative methods, including diffusion, direct neural generators, or latent-variable models.
These methods are simple and often effective on easy benchmarks, but they ignore the quotient geometry of pure states and may allocate capacity to extrinsic directions that have no physical meaning on $\mathbb{CP}^{d-1}$.
Our Euclidean diffusion and Euclidean flow baselines are representative of this ambient modeling perspective.
By contrast, intrinsic approaches define transport directly on $\mathbb{CP}^{d-1}$, where velocity fields live in tangent spaces and trajectories respect the manifold geometry by construction.
This ambient-versus-intrinsic distinction is central to our empirical comparison.

\paragraph{Flow matching for other quantum representations.}
QFM and ShadowFM are related in motivation but not direct experimental baselines for IFM.
QFM studies flow matching for density-matrix interpolation and circuit-based generation \citep{cui2025quantum}, whereas IFM targets distributions of pure-state rays on $\mathbb{CP}^{d-1}$.
Faithfully comparing these settings would require matching not only samples and metrics, but also the density-matrix objective, circuit ansatz, measurement model, and training protocol.
ShadowFM instead performs flow matching on classical shadows \citep{hahmshadowfm}, which are classical measurement-derived representations of quantum states rather than points on the pure-state projective manifold.
We therefore cite these methods as complementary related work and focus our controlled baselines on Euclidean FM, SSDM, Euclidean VP-SDE, QDDPM, QGAN, and the appendix generic geodesic FM comparator.

\paragraph{Optimal transport and geometric transport.}
Our formulation is also related to probability transport and optimal transport on structured spaces.
Recent work on Optimal Flow Matching shows in Euclidean space that, under suitable parameterizations and couplings, a single flow-matching objective can recover straight optimal-transport displacements \citep{tong2023conditional,tong2023improving,kerrigan2024dynamic}.
Although IFM does not solve a full optimal transport problem, its Monge--geodesic specialization can be viewed as a projective-manifold counterpart of displacement interpolation.
This connects our framework conceptually to displacement interpolation, continuity-equation-based density evolution, and action-minimizing flows \citep{cisneros2022distributed,chewi2021fast,nishino2025benamou}.

\paragraph{Geometric phase and Pancharatnam alignment.}
Because pure states are defined only up to global phase, geometric phase conventions are relevant even before one specifies a learning objective \citep{pancharatnam1956generalized,garza2023deciphering,ferrer2023topological,leinonen2023noncyclic,jisha2023waveguiding}.
In particular, the Pancharatnam in-phase condition provides a canonical way to compare two pure states by choosing representatives whose overlap is real and nonnegative \citep{pancharatnam1956generalized}.
Our phase-aligned path construction uses this convention as a gauge-fixing device that removes arbitrary phase mismatch before interpolation.
This is the main point at which IFM uses quantum-specific structure beyond generic manifold flow matching.

\paragraph{Relation to Riemannian flow matching on general geometries.}
Riemannian Flow Matching (RFM) provides a broad framework for flow matching on manifolds through user-specified premetrics and is simulation-free on simple geometries with closed-form geodesics \citep{chen2023flow}.
Our method shares this high-level conditional-velocity-matching philosophy, but the quantum pure-state setting introduces geometric structure that is not present in the standard simple-geometries catalog used by RFM.
The manifold is not a round sphere: $\mathbb{CP}^{d-1}$ is the quotient $S^{2d-1}/U(1)$ equipped with the Fubini--Study metric, with a nontrivial Hopf fibration, vertical gauge directions, and horizontal representatives \citep{mielnik1968geometry,bengtsson2017geometry}.
Consequently, comparing two endpoints requires fixing the relative phase before defining a chord or geodesic target; Pancharatnam alignment performs exactly this quotient-specific gauge fixing \citep{pancharatnam1956generalized}.
This step has no direct analogue on ordinary spheres, flat tori, hyperbolic spaces, meshes, or SPD manifolds.

The experimental regimes are also complementary.
RFM reports manifold experiments on settings such as $S^2$ earth and climate data, low-dimensional protein tori, high-dimensional flat tori with intrinsic dimension on the order of hundreds, mesh surfaces, and SPD matrices up to the $59\times 59$ case with intrinsic dimension $1770$ \citep{chen2023flow}.
By contrast, a dense $n$-qubit pure-state ensemble lives on $\mathbb{CP}^{2^n-1}$ with real dimension $2(2^n-1)$; the $12$-qubit experiments in this paper therefore correspond to real intrinsic dimension $8190$, and the $14$-qubit stress tests in Appendix~\ref{app:14q_flow_stress} reach dimension $32766$.
Thus, IFM is not merely an application of RFM to another named manifold: it develops the phase-aligned endpoint selection, horizontal velocity parameterization, and projective transport proofs needed for high-dimensional quantum pure-state ensembles.
The generic geodesic FM baseline in Appendix~\ref{app:fmgg_baseline} makes this distinction empirical by isolating the contribution of intrinsic geometry from the additional quantum-specific ingredients.

\paragraph{Physics-inspired transport viewpoints.}
Beyond machine-learning transport literature, our work is also physically inspired by quantum current and transport viewpoints in which ensemble evolution is described through continuity equations and associated flow fields \citep{delgado1999quantum,hodge2014electron}.
This includes, at a broad conceptual level, hydrodynamic, current-based, and Bohmian-style interpretations of quantum evolution, where dynamics are organized around probability transport rather than solely around noise or static wavefunction representations \citep{wyatt2005quantum,kuz1999quantum,tsekov2009bohmian}.
In this sense, the IFM velocity field can be viewed as an ensemble probability current on quantum state space.
We do not derive a new microscopic quantum law; rather, we use this perspective as motivation for an ensemble-level generative model on $\mathbb{CP}^{d-1}$.

\section{Additional Method Discussion and Algorithm}
\label{app:method_discussion_algorithm}

\subsection{Intrinsic versus Euclidean flow matching}

A major question in our study is whether the gains of IFM come simply from replacing diffusion with flow matching, or specifically from respecting the intrinsic geometry of quantum pure states.
To isolate these two effects, we compare IFM against a Euclidean flow matching baseline that performs transport in an ambient representation of the state vector and only enforces quantum-state validity through post hoc normalization, reflecting a common ambient modeling strategy for quantum-state generators \citep{lloyd2018quantum,dallaire2018quantum,liu2018differentiable,amin2018quantum}.
The Euclidean baseline therefore learns an \emph{extrinsic} transport field, whereas IFM learns an \emph{intrinsic} manifold transport field.

This comparison allows us to separate two distinct inductive biases:
\begin{equation}
    \text{diffusion vs.\ flow}
    \qquad\text{and}\qquad
    \text{ambient transport vs.\ intrinsic transport}.
\end{equation}
Our experiments show that both matter, but that intrinsic manifold transport is particularly important for faithfully modeling quantum pure-state ensembles.

\subsection{Why intrinsic flow matching scales better}
\label{sec:scalability}

A central motivation of IFM is scalability.
We argue that the difficulty of high-qubit quantum generative modeling is not only due to the large dimension of the state space, but also due to the algorithmic burden imposed by stochastic diffusion.
Diffusion-based models on $\mathbb{CP} ^{d-1}$ must estimate time-dependent score fields, rely on local diffusion approximations or teacher constructions, and generate samples through reverse-time stochastic simulation \citep{song2021scorebased,xu2026ssdm}.
Each of these steps introduces approximation error and optimization difficulty, which tend to accumulate as the dimension grows.

In contrast, IFM directly learns a transport velocity field.
This leads to three advantages.
First, the supervision signal is more direct: instead of learning local density gradients across noise scales, the model learns the direction in which probability mass should move.
Second, sampling is simpler: generation is performed by integrating a deterministic flow rather than reversing a stochastic process.
Third, the learned vector field is intrinsic to the tangent bundle of $\mathbb{CP} ^{d-1}$, avoiding unnecessary modeling effort in ambient directions that do not correspond to physical transport on the pure-state manifold.

Taken together, these properties make IFM a shorter and more stable generative pipeline.
Empirically, we find that this leads to improved robustness and stronger performance as the number of qubits increases, especially when compared against stochastic diffusion baselines and ambient Euclidean flow baselines.

\subsection{Relation to SSDM and Euclidean flow matching}
\label{sec:relation}

Our work is closely related to two existing generative perspectives: stochastic diffusion on quantum pure-state manifolds and Euclidean flow matching in ambient space.

\paragraph{Relation to SSDM.}
Stochastic Schr\"odinger diffusion models formulate quantum-state generation as a diffusion process on $\mathbb{CP} ^{d-1}$ and learn a score field associated with the resulting density evolution \citep{xu2026ssdm}.
Our method shares the same geometric setting, namely the pure-state manifold, but replaces score learning and reverse-time denoising by deterministic velocity matching and manifold flow integration.
In this sense, IFM replaces
\begin{equation}
    \text{diffusion + score matching + reverse SDE}
\end{equation}
with
\begin{equation}
    \text{intrinsic transport + velocity matching + ODE flow}.
\end{equation}
This change removes the need for local diffusion teachers and reverse-time stochastic simulation, while preserving the manifold-aware treatment of quantum pure states.

\paragraph{Relation to Euclidean flow matching.}
A second relevant baseline is Euclidean flow matching, which represents a quantum state as an ambient vector in $\mathbb{C}^d$ (or $\mathbb{R}^{2d}$), learns a transport field in that ambient space, and enforces state validity only through normalization or projection \citep{lipman2023flowmatchinggenerativemodeling,lipman2022flow}.
This approach replaces diffusion with flow, but does not respect the intrinsic geometry of $\mathbb{CP} ^{d-1}$.

IFM differs from Euclidean flow matching in that both its intermediate states and its velocity fields are defined intrinsically on the pure-state manifold.
This distinction is especially important in high dimensions, where ambient transport may allocate modeling capacity to extrinsic directions that are irrelevant to physical state evolution.
By comparing IFM with Euclidean flow matching, we isolate the effect of geometry from the effect of switching from diffusion to flow.

\paragraph{Positioning.}
Overall, our framework should be viewed as the deterministic flow-matching counterpart to SSDM on the same pure-state manifold: both are intrinsic, but they learn different dynamical objects and sample through different mechanisms.

\subsection{Neural parameterization and tangent projection}
\label{app:parameterization}

Sections~\ref{sec:conditional-paths} and~\ref{sec:population-characterization} identify the population target learned by IFM: a tangent velocity field on $\mathbb{CP}^{d-1}$ induced by the phase-aligned conditional path family.
The remaining question is architectural: how should a neural network represent such a field without reintroducing the ambient gauge redundancy that the previous sections worked to remove?

In practice, we represent a pure state $\psi \in \mathbb{C}^d$ through its real and imaginary parts and feed them, together with a time embedding of $t$, into a neural network.
The network first produces an ambient vector in $\mathbb{C}^d$, which is then projected onto the tangent space at $\psi$.
For the complex projective manifold, a valid projective velocity must be horizontal with respect to the quotient representation \citep{mielnik1968geometry,bengtsson2017geometry}, so we use the projection
\begin{equation}
    \Pi_\psi(v) = v - \psi \langle \psi, v \rangle,
\end{equation}
which satisfies $\langle \psi,\Pi_\psi(v)\rangle = 0$.
We therefore parameterize the learned field as
\begin{equation}
    v_\theta(\psi,t) := \Pi_\psi\!\big(f_\theta(\psi,t)\big),
\end{equation}
where $f_\theta$ is the unconstrained network output in the ambient space.
This is the implementation counterpart of the intrinsic theory: although the network computes in ambient coordinates, the object that enters the loss is always a tangent vector on projective state space.
As a result, the training problem remains aligned with the conditional target $u_t(\psi_t\mid\psi_0,\psi_1)$ from Section~\ref{sec:conditional-paths} and the marginal transport field characterized in Section~\ref{sec:population-characterization}.

\begin{proposition}[Horizontal representative as the unique minimal-norm projective velocity]
\label{prop:horizontal_minimal}
Fix a normalized representative $\psi \in \mathbb{C}^d$ and an ambient vector $w \in \mathbb{C}^d$.
Consider the affine family
\begin{equation}
    \mathcal{E}_{\psi}(w)
    :=
    \{
        w + a\psi + b\, i\psi : a,b \in \mathbb{R}
    \},
\end{equation}
which consists of all ambient first-order perturbations that induce the same projective first-order motion as $w$.
Then $\Pi_\psi(w)$ is the unique element of $\mathcal{E}_{\psi}(w)$ lying in $\mathcal{H}_\psi$, and it is the unique minimizer of the ambient norm over $\mathcal{E}_{\psi}(w)$:
\begin{equation}
    \Pi_\psi(w)
    =
    \arg\min_{\xi \in \mathcal{E}_{\psi}(w)} \|\xi\|_2.
\end{equation}
\end{proposition}

\noindent\textit{Proof.} See Appendix~\ref{app:proof_horizontal_minimal}.

Proposition~\ref{prop:horizontal_minimal} shows that the horizontal projection is the unique minimal-norm representative among ambient perturbations encoding the same projective motion, removing the redundant radial and pure-phase directions analytically rather than leaving them to the model.

\subsection{Sampling by deterministic manifold flow}
\label{sec:sampling}

After training, generation samples $\psi_0\sim p_0$ and integrates the learned tangent ODE, as in flow-matching and continuous-normalizing-flow samplers \citep{chen2018neural,lipman2023flowmatchinggenerativemodeling},
\begin{equation}
    \frac{d\psi_t}{dt} = v_\theta(\psi_t,t), \qquad v_\theta(\psi_t,t)\in \mathcal{H}_{\psi_t}.
\end{equation}
Numerically, we integrate in the ambient representative and retract back to the manifold after each step.
For $\xi\in\mathcal{H}_\psi$, we use the normalization-based retraction
\begin{equation}
    \mathrm{Retr}_{\psi}(\xi)
    :=
    \frac{\psi+\xi}{\|\psi+\xi\|_2},
\end{equation}
so an Euler step is
\begin{equation}
    \psi_{t_{k+1}}
    =
    \mathrm{Retr}_{\psi_{t_k}}
    \big(
        \Delta t\, v_\theta(\psi_{t_k},t_k)
    \big).
\end{equation}
This deterministic sampler directly transports base samples through the learned probability flow, rather than reversing a stochastic diffusion process \citep{song2021scorebased,ho2020denoising,xu2026ssdm}.

\begin{proposition}[First-order consistency of the retraction sampler]
\label{prop:retraction_consistency}
Assume $v_t(\psi)$ is continuously differentiable in $(\psi,t)$ and tangent to $\mathbb{CP}^{d-1}$.
Consider the intrinsic ODE
\begin{equation}
    \dot{\psi}_t = v_t(\psi_t)
\end{equation}
and the discrete update
\begin{equation}
    \psi_{t_{k+1}}
    =
    \mathrm{Retr}_{\psi_{t_k}}
    \big(
        \Delta t\, v_{t_k}(\psi_{t_k})
    \big).
\end{equation}
Then the normalization-based retraction $\mathrm{Retr}_\psi(\xi) = (\psi+\xi)/\|\psi+\xi\|_2$ is a valid first-order retraction on the unit-sphere representation, and the resulting Euler--retraction scheme is a first-order consistent integrator of the manifold flow.
\end{proposition}

\noindent\textit{Proof.} See Appendix~\ref{app:proof_retraction_consistency}.

Proposition~\ref{prop:retraction_consistency} connects the implementation to the continuous transport theory of Section~\ref{sec:population-characterization}: the normalization step is a first-order consistent discretization of the intrinsic manifold ODE.

\subsection{Algorithm}
\label{sec:algorithm}

Algorithm~\ref{alg:ifm} summarizes the training and sampling procedures of IFM.

\begin{algorithm}[t]
\caption{Intrinsic Flow Matching on $\mathbb{CP} ^{d-1}$}
\label{alg:ifm}
\begin{algorithmic}[1]
\STATE Initialize velocity network $v_\theta(\psi,t)$
\FOR{each training iteration}
    \STATE Sample base states $\psi_0 \sim p_0$
    \STATE Sample target states $\psi_1 \sim p_1$
    \STATE Sample times $t \sim \mathrm{Unif}[0,1]$
    \STATE Phase-align $\psi_1$ to $\bar{\psi}_1$
    \STATE Construct conditional path states
    \[
        \tilde{\psi}_t = (1-t)\psi_0 + t\bar{\psi}_1,
        \qquad
        \psi_t = \tilde{\psi}_t / \|\tilde{\psi}_t\|_2
    \]
    \STATE Compute target tangent velocities
    \[
        u_t(\psi_t \mid \psi_0,\psi_1)
        =
        \frac{1}{\|\tilde{\psi}_t\|_2}
        \Pi_{\psi_t}(\bar{\psi}_1-\psi_0)
    \]
    \STATE Predict manifold velocity $v_\theta(\psi_t,t)$
    \STATE Minimize
    \[
        \mathcal{L}_{\mathrm{IFM}}
        =
        \mathbb{E}\big[\|v_\theta(\psi_t,t)-u_t(\psi_t \mid \psi_0,\psi_1)\|^2_{\psi_t}\big]
    \]
\ENDFOR
\STATE \textbf{Sampling:}
\STATE Sample $\psi_0 \sim p_0$
\FOR{$k=0,\dots,K-1$}
    \STATE Evaluate tangent velocity $v_\theta(\psi_{t_k},t_k)$
    \STATE Update state by ODE step
    \[
        \psi_{t_{k+1}}
        \leftarrow
        \mathrm{Retr}_{\psi_{t_k}}
        \big(
            \Delta t\, v_\theta(\psi_{t_k},t_k)
        \big)
    \]
\ENDFOR
\STATE Return $\psi_1$
\end{algorithmic}
\end{algorithm}

In practice, $\mathrm{Retr}_{\psi}(\cdot)$ denotes a simple normalization-based retraction used to control numerical drift and keep the trajectory on the pure-state manifold.
This makes IFM straightforward to implement while preserving the geometric constraint of quantum pure states.

\section{Additional Proofs for the IFM Theory}

This appendix collects fuller proofs for the theoretical statements in the main text.
Throughout, we work with unit-norm representatives on the Hilbert sphere and identify tangent vectors on $\mathbb{CP}^{d-1}$ with horizontal vectors \citep{mielnik1968geometry,bengtsson2017geometry}
\(
\mathcal{H}_\psi = \{\xi \in \mathbb{C}^d : \langle \psi,\xi\rangle = 0\}
\).
Unless otherwise stated, all densities are with respect to the Fubini--Study volume form, and all vector fields are assumed measurable in space and continuous in time.

\subsection{Phase alignment and normalized chord paths}
\label{app:phase_alignment}

This subsection expands the phase-alignment step used in Section~\ref{sec:conditional-paths}, following the Pancharatnam in-phase convention for comparing pure-state representatives \citep{pancharatnam1956generalized,garza2023deciphering}.
The goal is to explain why the choice
\(
\bar{\psi}_1=\alpha(\psi_0,\psi_1)^{-1}\psi_1
\)
is the canonical representative for interpolation and why it leads to the shortest ambient chord among all representatives of the same projective endpoint.

\begin{proposition}[Pancharatnam-aligned representative]
\label{prop:pancharatnam_alignment}
Let $\psi_0,\psi_1 \in \mathbb{C}^d$ be unit vectors with $\langle \psi_0,\psi_1\rangle \neq 0$.
Then $\bar{\psi}_1=\alpha(\psi_0,\psi_1)^{-1}\psi_1$ is the unique representative of $[\psi_1]$ satisfying the Pancharatnam in-phase condition with respect to $\psi_0$,
\begin{equation}
    \langle \psi_0,\bar{\psi}_1\rangle = |\langle \psi_0,\psi_1\rangle| \in \mathbb{R}_{\ge 0},
\end{equation}
and equivalently it is the unique representative that maximizes the real overlap with $\psi_0$ and minimizes the ambient chord length $\|\psi_0-\bar{\psi}_1\|_2$ among all phase-equivalent representatives of $[\psi_1]$.
\end{proposition}

\begin{proposition}[Well-defined conditional manifold path]
\label{prop:path_welldefined}
Let $\psi_0,\psi_1 \in \mathbb{C}^d$ be unit-norm representatives of points in $\mathbb{CP}^{d-1}$, and let $\gamma_t(\psi_0,\psi_1)$ be defined as in Section~\ref{sec:conditional-paths} by phase alignment followed by normalized chord interpolation.
Then:
\begin{enumerate}
    \item the projective class $[\gamma_t(\psi_0,\psi_1)]$ is independent of the chosen representatives of $[\psi_0]$ and $[\psi_1]$;
    \item for every $t$ such that $\tilde{\psi}_t \neq 0$, the instantaneous velocity $u_t(\psi_t \mid \psi_0,\psi_1)$ belongs to the horizontal tangent space $\mathcal{H}_{\psi_t}$;
    \item the path satisfies the endpoint conditions $[\gamma_0(\psi_0,\psi_1)] = [\psi_0]$ and $[\gamma_1(\psi_0,\psi_1)] = [\psi_1]$.
\end{enumerate}
\end{proposition}

\begin{lemma}[Phase alignment maximizes real overlap]
\label{lem:phase_alignment_real_overlap}
Let $\psi_0,\psi_1 \in \mathbb{C}^d$ be unit vectors with $\langle \psi_0,\psi_1\rangle \neq 0$.
For any phase $\varphi \in \mathbb{R}$, define $\psi_1^{(\varphi)}:=e^{i\varphi}\psi_1$.
Then
\begin{equation}
    \max_{\varphi \in \mathbb{R}}
    \mathrm{Re}\,\langle \psi_0,\psi_1^{(\varphi)}\rangle
    =
    |\langle \psi_0,\psi_1\rangle|,
\end{equation}
and the maximizers are exactly the phases for which $\langle \psi_0,\psi_1^{(\varphi)}\rangle$ is real and nonnegative.
In particular, the canonical aligned representative
\(
\bar{\psi}_1=\alpha(\psi_0,\psi_1)^{-1}\psi_1
\)
satisfies
\(
\langle \psi_0,\bar{\psi}_1\rangle=|\langle \psi_0,\psi_1\rangle|.
\)
\end{lemma}

\begin{proof}
Write
\(
\langle \psi_0,\psi_1\rangle = r e^{i\beta}
\)
with $r=|\langle \psi_0,\psi_1\rangle|>0$.
Then
\[
\langle \psi_0,e^{i\varphi}\psi_1\rangle
=
r e^{i(\beta+\varphi)},
\]
so
\[
\mathrm{Re}\,\langle \psi_0,e^{i\varphi}\psi_1\rangle
=
r\cos(\beta+\varphi).
\]
This quantity is maximized when $\beta+\varphi \in 2\pi\mathbb{Z}$, in which case its value is $r$.
Equivalently, the maximizing representatives are exactly those for which the overlap is real and nonnegative.
Taking
\(
\varphi^\star=-\beta
\)
gives
\[
e^{i\varphi^\star}\psi_1
=
\frac{\overline{\langle \psi_0,\psi_1\rangle}}{|\langle \psi_0,\psi_1\rangle|}\psi_1
=
\alpha(\psi_0,\psi_1)^{-1}\psi_1
=
\bar{\psi}_1,
\]
which proves the claim.
\end{proof}

\begin{lemma}[Phase alignment minimizes ambient chord length]
\label{lem:phase_alignment_shortest_chord}
Under the assumptions of Lemma~\ref{lem:phase_alignment_real_overlap},
\begin{equation}
    \min_{\varphi \in \mathbb{R}}
    \|\psi_0 - e^{i\varphi}\psi_1\|_2^2
    =
    2 - 2|\langle \psi_0,\psi_1\rangle|,
\end{equation}
and the minimizers are exactly the phases that maximize
\(
\mathrm{Re}\,\langle \psi_0,e^{i\varphi}\psi_1\rangle
\).
Hence $\bar{\psi}_1$ is the representative of $[\psi_1]$ closest to $\psi_0$ on the unit sphere.
\end{lemma}

\begin{proof}
Since $\|\psi_0\|_2=\|\psi_1\|_2=1$, we have
\begin{align}
    \|\psi_0 - e^{i\varphi}\psi_1\|_2^2
    &=
    \langle \psi_0 - e^{i\varphi}\psi_1,\psi_0 - e^{i\varphi}\psi_1\rangle \\
    &=
    \|\psi_0\|_2^2 + \|\psi_1\|_2^2 - 2\,\mathrm{Re}\,\langle \psi_0,e^{i\varphi}\psi_1\rangle \\
    &=
    2 - 2\,\mathrm{Re}\,\langle \psi_0,e^{i\varphi}\psi_1\rangle.
\end{align}
Therefore minimizing the chord length is equivalent to maximizing the real part of the overlap.
Lemma~\ref{lem:phase_alignment_real_overlap} gives the optimal value and shows that $\bar{\psi}_1$ is a minimizer.
\end{proof}

\subsection{Interpolation paths used in the experiments}
\label{app:interpolation_details}

This subsection makes the experimental path choices explicit.
The same flow-matching regression objective can be paired with different conditional path families \citep{lipman2023flowmatchinggenerativemodeling,tong2023conditional,chen2023flow}, and the paper uses three such variants:
\begin{itemize}
    \item the phase-aligned normalized-chord path used by the main IFM model;
    \item a phase-aligned geodesic variant used as an ablation of our method;
    \item an appendix-only generic manifold-geodesic baseline that omits phase alignment and therefore excludes our quotient-specific representative selection.
\end{itemize}

\paragraph{Phase-aligned normalized-chord path.}
Given unit representatives $\psi_0,\psi_1 \in \mathbb{C}^d$, define the aligned endpoint
\[
\bar{\psi}_1 =
\begin{cases}
\alpha(\psi_0,\psi_1)^{-1}\psi_1, & \langle \psi_0,\psi_1\rangle \neq 0,\\
\psi_1, & \langle \psi_0,\psi_1\rangle = 0,
\end{cases}
\qquad
\alpha(\psi_0,\psi_1)=\frac{\langle \psi_0,\psi_1\rangle}{|\langle \psi_0,\psi_1\rangle|}.
\]
The ambient chord is
\[
\tilde{\psi}_t=(1-t)\psi_0+t\bar{\psi}_1,
\]
and the path used in the main method is its normalized version
\[
\psi_t=\frac{\tilde{\psi}_t}{\|\tilde{\psi}_t\|_2}.
\]
Differentiating the normalization map yields
\[
\dot{\psi}_t
=
\frac{1}{\|\tilde{\psi}_t\|_2}
\Big(I-\psi_t\psi_t^\ast\Big)(\bar{\psi}_1-\psi_0),
\]
which is the orthogonal projection of the ambient chord direction onto the tangent space of the unit sphere at $\psi_t$.
In the implementation, the final target is represented by its horizontal component
\(
\Pi_{\psi_t}(\xi)=\xi-\langle \psi_t,\xi\rangle \psi_t
\),
so that the learning target lies in $\mathcal{H}_{\psi_t}$ and therefore defines a projective tangent vector.

\paragraph{Phase-aligned geodesic interpolation (ours variant).}
Our geodesic ablation uses the same aligned endpoint $\bar{\psi}_1$ but replaces the normalized chord by the short spherical interpolation between $\psi_0$ and $\bar{\psi}_1$.
Let
\[
\theta = \arccos\!\big(\mathrm{clip}(\mathrm{Re}\,\langle \psi_0,\bar{\psi}_1\rangle,-1,1)\big).
\]
For nondegenerate angles, the path is
\[
\psi_t
=
\frac{\sin((1-t)\theta)}{\sin\theta}\,\psi_0
+
\frac{\sin(t\theta)}{\sin\theta}\,\bar{\psi}_1,
\]
with derivative
\[
\dot{\psi}_t
=
-\theta \frac{\cos((1-t)\theta)}{\sin\theta}\,\psi_0
+
\theta \frac{\cos(t\theta)}{\sin\theta}\,\bar{\psi}_1.
\]
For numerically small $\theta$, the implementation uses the corresponding linearized limit.
This is treated as an ablation of our method rather than as an unrelated baseline because it uses the same Pancharatnam-aligned representative selection and the same horizontal tangent representation \citep{pancharatnam1956generalized,bengtsson2017geometry}; only the interpolation rule is changed.

\paragraph{Generic manifold-geodesic baseline without phase alignment.}
Appendix~\ref{app:fmgg_baseline} also reports a closest in-code analogue of a generic manifold flow-matching baseline on $\mathbb{CP}^{d-1}$ \citep{chen2023flow,miller2024flowmm}.
In that variant, the raw representatives $\psi_0,\psi_1$ are used directly, without first replacing $\psi_1$ by the aligned representative $\bar{\psi}_1$.
Define
\[
\theta_{\mathrm{raw}}=\arccos\!\big(\mathrm{clip}(\mathrm{Re}\,\langle \psi_0,\psi_1\rangle,-1,1)\big),
\]
and set
\[
\psi_t^{\mathrm{raw}}
=
\frac{\sin((1-t)\theta_{\mathrm{raw}})}{\sin\theta_{\mathrm{raw}}}\,\psi_0
+
\frac{\sin(t\theta_{\mathrm{raw}})}{\sin\theta_{\mathrm{raw}}}\,\psi_1,
\]
again with the small-angle linear limit when needed.
This path is geometrically valid on the sphere representation, but because it omits phase alignment it does not quotient out arbitrary endpoint phase mismatch before interpolation.
That is precisely why we use it only as an appendix baseline and not as the main IFM construction.

\subsection{Proof of Proposition~\ref{prop:pancharatnam_alignment}}
\label{app:proof_pancharatnam_alignment}

\begin{proof}
The first claim is exactly the Pancharatnam in-phase condition.
By Lemma~\ref{lem:phase_alignment_real_overlap}, among all phase-equivalent representatives
\(
\psi_1^{(\varphi)}=e^{i\varphi}\psi_1
\),
the representative
\(
\bar{\psi}_1=\alpha(\psi_0,\psi_1)^{-1}\psi_1
\)
is characterized by
\[
\langle \psi_0,\bar{\psi}_1\rangle
=
|\langle \psi_0,\psi_1\rangle|
\in \mathbb{R}_{\ge 0},
\]
and is the unique phase choice with this property when $\langle \psi_0,\psi_1\rangle \neq 0$.
This proves the Pancharatnam-alignment statement.

Again by Lemma~\ref{lem:phase_alignment_real_overlap}, the same representative uniquely maximizes
\(
\mathrm{Re}\,\langle \psi_0,e^{i\varphi}\psi_1\rangle
\)
over all $\varphi$.
Lemma~\ref{lem:phase_alignment_shortest_chord} shows that minimizing
\(
\|\psi_0-e^{i\varphi}\psi_1\|_2
\)
is equivalent to maximizing this same real overlap.
Hence $\bar{\psi}_1$ is also the unique representative minimizing the ambient chord length among all representatives of the projective endpoint $[\psi_1]$.

Therefore the three characterizations are equivalent:
\begin{enumerate}
    \item satisfying the Pancharatnam in-phase condition with respect to $\psi_0$;
    \item maximizing the real overlap with $\psi_0$;
    \item minimizing the ambient chord length to $\psi_0$.
\end{enumerate}
This proves Proposition~\ref{prop:pancharatnam_alignment}.
\end{proof}

\begin{remark}
The point of Lemmas~\ref{lem:phase_alignment_real_overlap} and~\ref{lem:phase_alignment_shortest_chord} is not that projective geometry is reduced to ambient Euclidean geometry.
Rather, once one chooses representatives in the quotient space $\mathbb{CP}^{d-1}$, there remains a gauge degree of freedom.
Phase alignment fixes that freedom in the way that produces the shortest ambient interpolation consistent with the same projective endpoints, thereby removing spurious motion due solely to global phase.
\end{remark}

\subsection{Proof of Proposition~\ref{prop:path_welldefined}}
\label{app:proof_path_welldefined}

\begin{proof}
Let $\psi_0' = e^{i\theta_0}\psi_0$ and $\psi_1' = e^{i\theta_1}\psi_1$ be alternative representatives of the same projective points.
Then
\[
\alpha(\psi_0',\psi_1')
=
\frac{\langle e^{i\theta_0}\psi_0, e^{i\theta_1}\psi_1\rangle}{|\langle e^{i\theta_0}\psi_0, e^{i\theta_1}\psi_1\rangle|}
=
e^{i(\theta_1-\theta_0)}\alpha(\psi_0,\psi_1).
\]
Hence
\[
\bar{\psi}_1'
=
\alpha(\psi_0',\psi_1')^{-1}\psi_1'
=
e^{i\theta_0}\alpha(\psi_0,\psi_1)^{-1}\psi_1
=
e^{i\theta_0}\bar{\psi}_1.
\]
The interpolant therefore transforms as
\[
\tilde{\psi}_t'
=
(1-t)\psi_0' + t\bar{\psi}_1'
=
e^{i\theta_0}\big((1-t)\psi_0 + t\bar{\psi}_1\big)
=
e^{i\theta_0}\tilde{\psi}_t.
\]
After normalization, $\gamma_t(\psi_0',\psi_1') = e^{i\theta_0}\gamma_t(\psi_0,\psi_1)$, so both define the same point in $\mathbb{CP}^{d-1}$.
This proves representative-independence.

The endpoint property follows immediately:
\[
\gamma_0(\psi_0,\psi_1)=\psi_0,
\qquad
\gamma_1(\psi_0,\psi_1)=\bar{\psi}_1,
\]
and $[\bar{\psi}_1]=[\psi_1]$ because they differ by a unit-modulus phase.

For tangency, define
\(
w_t := \bar{\psi}_1-\psi_0
\)
so that $\dot{\tilde{\psi}}_t = w_t$.
Differentiating the normalized lift $\psi_t=\tilde{\psi}_t/\|\tilde{\psi}_t\|_2$ gives
\[
\dot{\psi}_t
=
\frac{w_t}{\|\tilde{\psi}_t\|_2}
-
\frac{\tilde{\psi}_t\, \mathrm{Re}\langle \tilde{\psi}_t,w_t\rangle}{\|\tilde{\psi}_t\|_2^3}.
\]
In particular, $\dot{\psi}_t$ is tangent to the unit sphere, i.e.,
\(
\mathrm{Re}\langle \psi_t,\dot{\psi}_t\rangle = 0
\).
To pass from the sphere to the projective quotient, one removes the residual vertical phase component along $i\psi_t$ and retains the horizontal representative. By definition this representative is
\[
u_t(\psi_t\mid\psi_0,\psi_1)
:=
\frac{1}{\|\tilde{\psi}_t\|_2}\Pi_{\psi_t}(w_t).
\]
Since
\[
\langle \psi_t,\Pi_{\psi_t}(w_t)\rangle
=
\langle \psi_t,w_t-\psi_t\langle\psi_t,w_t\rangle\rangle
=
0,
\]
we conclude that $u_t(\psi_t\mid\psi_0,\psi_1)\in\mathcal{H}_{\psi_t}$.
\end{proof}

\subsection{Proof of Theorem~\ref{thm:ifm_population_opt}}
\label{app:proof_population_opt}

\begin{proof}
Fix $t \in (0,1)$ and abbreviate
\(
U := u_t(\psi_t \mid \psi_0,\psi_1)\in T_{\psi_t}\mathbb{CP}^{d-1}.
\)
For any measurable tangent field $v(\cdot,t)$, the population risk is
\[
\mathcal{J}_t[v]
=
\mathbb{E}\big[\|v(\psi_t,t)-U\|_{\psi_t}^2\big].
\]
Condition on the intermediate state $\psi_t$:
\[
\mathcal{J}_t[v]
=
\mathbb{E}_{\psi_t}\left[
    \mathbb{E}\big[
        \|v(\psi_t,t)-U\|_{\psi_t}^2
        \,\big|\,
        \psi_t
    \big]
\right].
\]
Thus the minimization decouples pointwise in $\psi$.
Fix $\psi$ and write $m(\psi):=\mathbb{E}[U\mid \psi_t=\psi]$.
Using the Hilbert-space identity
\(
\|a-b\|^2 = \|a-c\|^2 + \|b-c\|^2 - 2\langle a-c,b-c\rangle
\)
with $c=m(\psi)$ and then taking conditional expectation, we obtain
\[
\mathbb{E}\big[
    \|v(\psi,t)-U\|_{\psi}^2
    \,\big|\,
    \psi_t=\psi
\big]
=
\|v(\psi,t)-m(\psi)\|_{\psi}^2
+
\mathbb{E}\big[
    \|U-m(\psi)\|_{\psi}^2
    \,\big|\,
    \psi_t=\psi
\big].
\]
The second term is independent of $v$, so the unique minimizer is $v(\psi,t)=m(\psi)$.
Since this holds for $p_t$-almost every $\psi$, we obtain
\[
v_t^\star(\psi)
=
\mathbb{E}\big[
    u_t(\psi_t \mid \psi_0,\psi_1)
    \,\big|\,
    \psi_t=\psi
\big].
\]
\end{proof}

\subsection{Proof of Proposition~\ref{prop:ifm_loss_decomposition}}
\label{app:proof_loss_decomposition}

\begin{proof}
Let
\(
\bar{u}_t(\psi)
:=
\mathbb{E}[u_t(\psi_t\mid\psi_0,\psi_1)\mid \psi_t=\psi].
\)
For notational brevity write
\(
U:=u_t(\psi_t\mid\psi_0,\psi_1)
\)
and
\(
\bar{U}:=\bar{u}_t(\psi_t).
\)
Then
\[
v(\psi_t,t)-U
=
\big(v(\psi_t,t)-\bar{U}\big) + \big(\bar{U}-U\big).
\]
Expanding the squared norm and taking expectation gives
\begin{align}
\mathcal{J}_t[v]
&=
\mathbb{E}\big[\|v(\psi_t,t)-\bar{U}\|_{\psi_t}^2\big]
+
\mathbb{E}\big[\|U-\bar{U}\|_{\psi_t}^2\big] \nonumber\\
&\quad
+ 2\,\mathbb{E}\big[
\langle v(\psi_t,t)-\bar{U}, \bar{U}-U\rangle_{\psi_t}
\big].
\label{eq:appendix_ifm_pythagorean}
\end{align}
It remains to show that the cross term vanishes.
Conditioning on $\psi_t$ yields
\[
\mathbb{E}\big[
\langle v(\psi_t,t)-\bar{U}, \bar{U}-U\rangle_{\psi_t}
\big]
=
\mathbb{E}\Big[
\big\langle
v(\psi_t,t)-\bar{U},
\mathbb{E}[\bar{U}-U \mid \psi_t]
\big\rangle_{\psi_t}
\Big].
\]
By definition of $\bar{U}$,
\(
\mathbb{E}[U\mid \psi_t]=\bar{U}
\),
hence
\(
\mathbb{E}[\bar{U}-U\mid \psi_t]=0
\).
Therefore the last term in \eqref{eq:appendix_ifm_pythagorean} is zero, and the desired decomposition follows.
Since the second term does not depend on $v$, the conditional and marginal objectives have identical minimizers.
\end{proof}

\subsection{Proof of Proposition~\ref{prop:ifm_continuity}}
\label{app:proof_continuity}

\begin{proof}
Let $\varphi \in C^\infty(\mathbb{CP}^{d-1})$ be a smooth test function.
Since $t \mapsto \gamma_t(\psi_0,\psi_1)$ is $C^1$, the chain rule gives
\[
\frac{d}{dt}\varphi(\psi_t)
=
\left\langle
    \nabla_{\mathrm{FS}}\varphi(\psi_t),
    u_t(\psi_t\mid\psi_0,\psi_1)
\right\rangle_{\psi_t}.
\]
Taking expectation yields
\[
\frac{d}{dt}\mathbb{E}[\varphi(\psi_t)]
=
\mathbb{E}\left[
    \left\langle
        \nabla_{\mathrm{FS}}\varphi(\psi_t),
        u_t(\psi_t\mid\psi_0,\psi_1)
    \right\rangle_{\psi_t}
\right].
\]
Now condition on $\psi_t$ and use Theorem~\ref{thm:ifm_population_opt}:
\[
\frac{d}{dt}\mathbb{E}[\varphi(\psi_t)]
=
\mathbb{E}\left[
    \left\langle
        \nabla_{\mathrm{FS}}\varphi(\psi_t),
        v_t^\star(\psi_t)
    \right\rangle_{\psi_t}
\right].
\]
Writing the expectation with respect to the density $p_t$ gives
\[
\frac{d}{dt}\int_{\mathbb{CP}^{d-1}} \varphi(\psi)\, p_t(\psi)\, d\mathrm{vol}_{\mathrm{FS}}(\psi)
=
\int_{\mathbb{CP}^{d-1}}
\langle \nabla_{\mathrm{FS}}\varphi(\psi), v_t^\star(\psi)\rangle_{\psi}\,
p_t(\psi)\, d\mathrm{vol}_{\mathrm{FS}}(\psi).
\]
By the definition of the Riemannian divergence,
\[
\int_{\mathbb{CP}^{d-1}}
\langle \nabla_{\mathrm{FS}}\varphi, v_t^\star\rangle\, p_t\, d\mathrm{vol}_{\mathrm{FS}}
=
-
\int_{\mathbb{CP}^{d-1}}
\varphi\, \mathrm{div}_{\mathbb{CP}}(p_t v_t^\star)\, d\mathrm{vol}_{\mathrm{FS}},
\]
which is exactly the weak formulation of
\(
\partial_t p_t + \mathrm{div}_{\mathbb{CP}}(p_t v_t^\star)=0
\).
Since this identity holds for every smooth test function $\varphi$, the claim follows.
\end{proof}

\subsection{Proof of Theorem~\ref{thm:ifm_endpoint}}
\label{app:proof_endpoint}

\begin{proof}
By the local Lipschitz assumption on $v_t^\star$ and compactness of $\mathbb{CP}^{d-1}$, the manifold ODE $\dot{\psi}_t=v_t^\star(\psi_t)$ admits a unique global flow $\Phi_{0,t}$ on $[0,1]$ \citep{lee2012smooth}.
Define the pushforward $\tilde p_t := (\Phi_{0,t})_\# p_0$.
For any test function $\varphi\in C^\infty(\mathbb{CP}^{d-1})$, the chain rule along the deterministic flow gives
\[
\frac{d}{dt}\int_{\mathbb{CP}^{d-1}} \varphi\, d\tilde p_t
=
\mathbb{E}\big[\langle \nabla_{\mathrm{FS}}\varphi(\Phi_{0,t}\psi_0), v_t^\star(\Phi_{0,t}\psi_0)\rangle_{\Phi_{0,t}\psi_0}\big]
=
\int_{\mathbb{CP}^{d-1}} \langle \nabla_{\mathrm{FS}}\varphi, v_t^\star\rangle\, d\tilde p_t,
\]
so $\tilde p_t$ satisfies the manifold continuity equation $\partial_t \tilde p_t + \mathrm{div}_{\mathbb{CP}}(\tilde p_t v_t^\star)=0$ in the weak sense, with $\tilde p_0=p_0$.
By Proposition~\ref{prop:ifm_continuity}, the family $p_t$ defined as the law of $\gamma_t(\psi_0,\psi_1)$ also satisfies this continuity equation with $p_0$ as its initial condition.
The Lipschitz regularity of $v_t^\star$ on the compact manifold $\mathbb{CP}^{d-1}$ implies uniqueness of weak solutions to the continuity equation in the standard transport-equation sense \citep{ambrosio2008transport,villani2008stability}, hence $\tilde p_t = p_t$ for all $t\in[0,1]$.
The boundary condition $\gamma_1(\psi_0,\psi_1)=\psi_1$ a.s.\ gives $p_1$ as the marginal at $t=1$, so $(\Phi_{0,1})_\# p_0 = p_1$.
\end{proof}

\begin{remark}[Dimension dependence of the Lipschitz constant]
\label{rem:lipschitz_d}
Compactness of $\mathbb{CP}^{d-1}$ removes the linear-growth condition needed for Cauchy--Lipschitz on $\mathbb{R}^N$, but it does not by itself bound the Lipschitz constant $L = L(v_t^\star)$ in any $d$-uniform way.
For the normalized chord conditional path of Section~\ref{sec:conditional-paths},
\[
u_t(\psi_t\mid\psi_0,\psi_1)
=
\frac{1}{\|(1-t)\psi_0 + t\bar\psi_1\|_2}\,\Pi_{\psi_t}(\bar\psi_1 - \psi_0),
\]
whose smoothness in $(\psi_0,\psi_1,t)$ degrades when $|\langle\psi_0,\psi_1\rangle|\to 0$: the Pancharatnam phase $\alpha(\psi_0,\psi_1)=\langle\psi_0,\psi_1\rangle/|\langle\psi_0,\psi_1\rangle|$ is then ambiguous, and $\|\tilde\psi_t\|_2$ approaches its lower envelope $\sqrt{(1-t)^2+t^2}$ where the radial denominator is least cushioned.
Random pure states on $\mathbb{CP}^{d-1}$ are near-orthogonal with overwhelming probability for large $d$ (typical overlap $\sim d^{-1/2}$), so for couplings whose support reaches the near-orthogonal regime the Lipschitz constant $L$ can grow with $d$.
A clean sufficient condition for a $d$-independent bound is a margin
\begin{equation}
\label{eq:overlap_margin}
\inf_{(\psi_0,\psi_1)\in\mathrm{supp}(\pi)}|\langle\psi_0,\psi_1\rangle| \;\ge\; \kappa \;>\; 0,
\end{equation}
under which a direct calculation gives $L = O(\kappa^{-c})$ for a small numerical constant $c$ that we do not optimize.
Without such a margin, the constants in Theorem~\ref{thm:ifm_endpoint} and Proposition~\ref{prop:ifm_stability} are best interpreted as data-dependent rather than dimension-uniform.
Margins of the form \eqref{eq:overlap_margin} can arise from clustered targets or from smoothed conditional path families near orthogonality.
\end{remark}

\subsection{Proof of Proposition~\ref{prop:ifm_stability}}
\label{app:proof_stability}

The proof proceeds in three steps: a pointwise coupling bound, an $L^2$ aggregation along the true flow, and the transition from coupling cost to Wasserstein-$2$ distance.
The pointwise step is isolated as a lemma so that the change-of-measure issue between $p_t$ and $\hat p_t$ is addressed explicitly rather than absorbed into a constant.

\begin{lemma}[Pointwise coupling stability under Lipschitz flows]
\label{lem:pointwise_coupling_stability}
Let $v^\star, \hat v: \mathbb{CP}^{d-1}\times[0,1]\to T\mathbb{CP}^{d-1}$ be tangent vector fields with flows $\Phi_{0,t},\hat\Phi_{0,t}$ from a common initial point $\psi_0$ in the unit-sphere representation, and assume $\hat v$ is $L$-Lipschitz on $\mathbb{CP}^{d-1}$ uniformly in $t\in[0,1]$.
Set $\psi_t^\star := \Phi_{0,t}(\psi_0)$, $\hat\psi_t := \hat\Phi_{0,t}(\psi_0)$, $\delta_t := \psi_t^\star - \hat\psi_t$, and
\begin{equation}
\label{eq:stability_velocity_error_true}
e_t \;:=\; \big\|v_t^\star(\psi_t^\star) - \hat v_t(\psi_t^\star)\big\|_2,
\end{equation}
the velocity error evaluated along the \emph{true} trajectory $\psi^\star_t$.
Then for every $t\in[0,1]$,
\begin{equation}
\label{eq:stability_pointwise}
\|\delta_t\|_2 \;\le\; \int_0^t e^{L(t-s)}\,e_s\, ds.
\end{equation}
\end{lemma}

\begin{proof}[Proof of Lemma~\ref{lem:pointwise_coupling_stability}]
The decisive step is to decompose the velocity difference at the common evaluation point $\psi_t^\star$ rather than at $\hat\psi_t$:
\[
\frac{d}{dt}\delta_t
\;=\;
v_t^\star(\psi_t^\star) - \hat v_t(\hat\psi_t)
\;=\;
\underbrace{\big[v_t^\star(\psi_t^\star) - \hat v_t(\psi_t^\star)\big]}_{\text{velocity error along }p_t}
\;+\;
\underbrace{\big[\hat v_t(\psi_t^\star) - \hat v_t(\hat\psi_t)\big]}_{\text{Lipschitz term in }\delta_t}.
\]
The Lipschitz assumption on $\hat v_t$ bounds the second bracket by $L\|\delta_t\|_2$, giving
\[
\bigg\|\frac{d}{dt}\delta_t\bigg\|_2
\;\le\;
e_t + L\,\|\delta_t\|_2.
\]
Since $\frac{d}{dt}\|\delta_t\|_2 \le \|\frac{d}{dt}\delta_t\|_2$ wherever $\delta_t\neq 0$, with $\delta_0=0$, Gr\"onwall's inequality in integral form yields \eqref{eq:stability_pointwise}.
The point of the decomposition is that the velocity error appears at $\psi_t^\star$, whose law is exactly $p_t$ by Theorem~\ref{thm:ifm_endpoint}; no Radon--Nikodym factor between $p_t$ and $\hat p_t$ enters the bound.
\end{proof}

\begin{proof}[Proof of Proposition~\ref{prop:ifm_stability}]
Couple the two flows by sharing the initial sample $\psi_0\sim p_0$, and use the notation of Lemma~\ref{lem:pointwise_coupling_stability}.
By Theorem~\ref{thm:ifm_endpoint}, $\psi_1^\star\sim p_1$, while $\hat\psi_1\sim \hat p_1$ by definition of $\hat\Phi_{0,1}$.

\smallskip
\noindent\textbf{Step 1 (pointwise bound).}
Lemma~\ref{lem:pointwise_coupling_stability} gives, at $t=1$,
\[
\|\delta_1\|_2 \;\le\; \int_0^1 e^{L(1-s)}\, e_s\, ds.
\]
Squaring and applying Cauchy--Schwarz in the time variable,
\[
\|\delta_1\|_2^2
\;\le\;
\bigg(\int_0^1 e^{L(1-s)}\, ds\bigg)
\bigg(\int_0^1 e^{L(1-s)}\, e_s^2\, ds\bigg)
\;\le\;
e^{2L}\int_0^1 e_s^2\, ds.
\]

\smallskip
\noindent\textbf{Step 2 ($L^2$ aggregation along $p_t$).}
Taking expectation over $\psi_0\sim p_0$ and using Theorem~\ref{thm:ifm_endpoint} to identify the law of $\psi_s^\star$ as $p_s$, the velocity error $e_s$ becomes the integrand appearing in the proposition's assumption:
\[
\mathbb{E}_{\psi_0\sim p_0}\,\|\delta_1\|_2^2
\;\le\;
e^{2L}\int_0^1 \mathbb{E}_{\psi_0\sim p_0}\,\big\|v_s^\star(\psi_s^\star) - \hat v_s(\psi_s^\star)\big\|_2^2\, ds
\;=\;
e^{2L}\int_0^1 \mathbb{E}_{\psi\sim p_s}\,\big\|v_s^\star(\psi) - \hat v_s(\psi)\big\|_2^2\, ds
\;\le\;
e^{2L}\,\varepsilon^2.
\]
The horizontal-norm and ambient-norm coincide on the horizontal subspace, so $\|\,\cdot\,\|_\psi$ and $\|\,\cdot\,\|_2$ agree under the integrand.
This step is where the careful evaluation point in Lemma~\ref{lem:pointwise_coupling_stability} pays off: no change-of-measure between $\hat p_s$ and $p_s$ is invoked, and no Radon--Nikodym constant has to be tracked.

\smallskip
\noindent\textbf{Step 3 (coupling cost upper-bounds $W_2$).}
On $\mathbb{CP}^{d-1}$, the Fubini--Study geodesic distance is dominated by the ambient chord distance in any unit-norm representatives \citep{bengtsson2017geometry}: $d_{\mathrm{FS}}([\psi],[\phi]) \le \|\psi - \phi\|_2$.
Hence the coupling $(\psi_1^\star,\hat\psi_1)$ is an admissible plan between $p_1$ and $\hat p_1$, and
\[
W_2^{\mathrm{FS}}\!\big(\hat p_1,\, p_1\big)^2
\;\le\;
\mathbb{E}\, d_{\mathrm{FS}}([\psi_1^\star],[\hat\psi_1])^2
\;\le\;
\mathbb{E}\,\|\delta_1\|_2^2
\;\le\;
e^{2L}\,\varepsilon^2,
\]
giving $W_2^{\mathrm{FS}}(\hat p_1, p_1) \le e^{L}\,\varepsilon$ as claimed.
\end{proof}

\subsection{Proof of Proposition~\ref{prop:physical_compatibility}}
\label{app:proof_physical_compatibility}

\begin{proof}
Let $\varphi \in C^\infty(\mathbb{CP}^{d-1})$ be a smooth test function and let $\psi_t=\Phi_{0,t}(\psi_0)$ with $\psi_0\sim p_0$.
By the chain rule along the deterministic flow generated by $X_t$,
\[
\frac{d}{dt}\varphi(\psi_t)
=
\langle \nabla_{\mathrm{FS}}\varphi(\psi_t), X_t(\psi_t)\rangle_{\psi_t}.
\]
Taking expectation gives
\[
\frac{d}{dt}\mathbb{E}[\varphi(\psi_t)]
=
\mathbb{E}\big[
\langle \nabla_{\mathrm{FS}}\varphi(\psi_t), X_t(\psi_t)\rangle_{\psi_t}
\big].
\]
Since $p_t=(\Phi_{0,t})_\# p_0$ is the law of $\psi_t$, we may rewrite this as
\[
\frac{d}{dt}\int_{\mathbb{CP}^{d-1}} \varphi(\psi)\,p_t(\psi)\,d\mathrm{vol}_{\mathrm{FS}}(\psi)
=
\int_{\mathbb{CP}^{d-1}}
\langle \nabla_{\mathrm{FS}}\varphi(\psi),X_t(\psi)\rangle_\psi\,
p_t(\psi)\,d\mathrm{vol}_{\mathrm{FS}}(\psi).
\]
Integrating by parts on the manifold yields
\[
\int_{\mathbb{CP}^{d-1}}
\langle \nabla_{\mathrm{FS}}\varphi, X_t\rangle\,p_t\,d\mathrm{vol}_{\mathrm{FS}}
=
-
\int_{\mathbb{CP}^{d-1}}
\varphi\,\mathrm{div}_{\mathbb{CP}}(p_t X_t)\,d\mathrm{vol}_{\mathrm{FS}},
\]
which proves the weak continuity equation
\(
\partial_t p_t+\mathrm{div}_{\mathbb{CP}}(p_tX_t)=0
\).

For the second claim, assume the IFM conditional path family is chosen so that
\(
\gamma_t(\psi_0,\psi_1)=\Phi_{0,t}(\psi_0)
\)
almost surely with
\(
\psi_1=\Phi_{0,1}(\psi_0)
\).
Then along every conditional path,
\[
u_t(\psi_t\mid\psi_0,\psi_1)
=
\frac{d}{dt}\gamma_t(\psi_0,\psi_1)
=
\frac{d}{dt}\Phi_{0,t}(\psi_0)
=
X_t(\psi_t).
\]
Applying Theorem~\ref{thm:ifm_population_opt}, the population IFM minimizer is
\[
v_t^\star(\psi)
=
\mathbb{E}\big[
u_t(\psi_t\mid\psi_0,\psi_1)\mid \psi_t=\psi
\big]
=
\mathbb{E}\big[
X_t(\psi_t)\mid \psi_t=\psi
\big]
=
X_t(\psi)
\]
for $p_t$-almost every $\psi$.
This proves exact recovery of the deterministic projective flow field.
\end{proof}

\subsection{Proof of Proposition~\ref{prop:ifm_monge_geodesic}}
\label{app:proof_monge_geodesic}

\begin{proof}
Under the deterministic coupling assumption,
\(
\psi_1 = T(\psi_0)
\)
almost surely, so the conditional velocity is
\[
U
=
u_t(\psi_t\mid\psi_0,\psi_1)
=
\frac{d}{dt}\gamma_t(\psi_0,T(\psi_0))
\]
almost surely.
By Theorem~\ref{thm:ifm_population_opt},
\[
v_t^\star(\psi)
=
\mathbb{E}[U\mid \psi_t=\psi].
\]
Now fix $t\in(0,1)$ and suppose $\psi=\psi_t$ lies in the full-measure image of the injective map
\(
F_t(\psi_0):=\gamma_t(\psi_0,T(\psi_0)).
\)
Then there is a unique preimage $\psi_0 = F_t^{-1}(\psi)$.
Consequently, conditioning on $\psi_t=\psi$ collapses the conditional expectation:
\[
\mathbb{E}[U\mid \psi_t=\psi]
=
\frac{d}{dt}\gamma_t(F_t^{-1}(\psi),T(F_t^{-1}(\psi))).
\]
Evaluating at $\psi=\psi_t=F_t(\psi_0)$ yields
\[
v_t^\star(\psi_t)
=
\frac{d}{dt}\gamma_t(\psi_0,T(\psi_0))
\]
for $p_t$-almost every $\psi_t$.
This proves exact recovery of the deterministic velocity field of the chosen geodesic interpolation.

If, in addition, $T$ is the optimal Monge map for the quadratic Fubini--Study cost, then the curve
\(
t\mapsto \gamma_t(\psi_0,T(\psi_0))
\)
is precisely the displacement interpolation associated with $T$.
Hence the recovered velocity field is the Benamou--Brenier velocity field of that optimal interpolation.
\end{proof}

\subsection{Proof of Proposition~\ref{prop:horizontal_minimal}}
\label{app:proof_horizontal_minimal}

\begin{proof}
Fix $\psi$ with $\|\psi\|_2=1$ and define
\[
\mathcal{E}_{\psi}(w)
=
\{
w+a\psi+b\,i\psi : a,b\in\mathbb{R}
\}.
\]
The directions $\psi$ and $i\psi$ span the radial and phase-vertical directions of the unit-sphere representation at $\psi$.
Perturbations that differ by an element of $\mathrm{span}_{\mathbb{R}}\{\psi,i\psi\}$ therefore have the same first-order effect on the underlying projective point.

We first show existence and uniqueness of the horizontal representative.
Let
\[
\xi = w+a\psi+b\,i\psi \in \mathcal{E}_{\psi}(w).
\]
The horizontality condition $\langle \psi,\xi\rangle=0$ becomes
\[
\langle \psi,w\rangle + a + bi = 0.
\]
Since $a,b\in\mathbb{R}$, this linear equation has the unique solution
\[
a = -\mathrm{Re}\langle \psi,w\rangle,
\qquad
b = -\mathrm{Im}\langle \psi,w\rangle.
\]
Substituting back gives
\[
\xi
=
w-\mathrm{Re}\langle \psi,w\rangle\,\psi-\mathrm{Im}\langle \psi,w\rangle\, i\psi
=
w-\psi\langle \psi,w\rangle
=
\Pi_\psi(w).
\]
Thus $\Pi_\psi(w)$ is the unique element of $\mathcal{E}_{\psi}(w)$ lying in $\mathcal{H}_\psi$.

It remains to prove the minimal-norm property.
Write any $\xi \in \mathcal{E}_{\psi}(w)$ as
\[
\xi = \Pi_\psi(w) + a\psi + b\,i\psi
\]
with $a,b\in\mathbb{R}$.
Because $\Pi_\psi(w)\in \mathcal{H}_\psi$, it is orthogonal to both $\psi$ and $i\psi$ under the ambient Hermitian inner product.
Therefore
\[
\|\xi\|_2^2
=
\|\Pi_\psi(w)\|_2^2 + a^2 + b^2.
\]
This quantity is minimized if and only if $a=b=0$, which proves that $\Pi_\psi(w)$ is the unique minimizer of the ambient norm over $\mathcal{E}_{\psi}(w)$.
\end{proof}

\subsection{Proof of Proposition~\ref{prop:retraction_consistency}}
\label{app:proof_retraction_consistency}

\begin{proof}
The argument has two steps: a local one-step truncation analysis (showing first-order consistency of the Euler--retraction map) and a global accumulation argument on $\mathbb{CP}^{d-1}$ (showing that the discrete trajectory converges to the exact flow at first order in $\Delta t$).

\smallskip
\noindent\textbf{Step 1: First-order consistency of the retraction.}
Fix a unit-norm representative $\psi$ and a tangent increment $\xi \in \mathcal{H}_\psi$.
By definition,
\[
\mathrm{Retr}_\psi(\xi)
=
\frac{\psi+\xi}{\|\psi+\xi\|_2}.
\]
Clearly $\mathrm{Retr}_\psi(0)=\psi$.
To verify the differential condition for a first-order retraction, consider
\[
R(\epsilon)
:=
\mathrm{Retr}_\psi(\epsilon \xi)
=
\frac{\psi+\epsilon \xi}{\|\psi+\epsilon \xi\|_2}.
\]
Because $\langle \psi,\xi\rangle = 0$, we have
\[
\|\psi+\epsilon \xi\|_2^2
=
1 + \epsilon^2\|\xi\|_2^2,
\qquad
\|\psi+\epsilon \xi\|_2
=
1 + O(\epsilon^2).
\]
Therefore
\[
R(\epsilon)
=
(\psi+\epsilon \xi)(1+O(\epsilon^2))
=
\psi+\epsilon\xi+O(\epsilon^2),
\]
so
\[
\left.\frac{d}{d\epsilon}R(\epsilon)\right|_{\epsilon=0}
=
\xi.
\]
Thus $\mathrm{Retr}_\psi$ is a valid first-order retraction on the unit-sphere representation.

Now consider the exact manifold ODE $\dot{\psi}_t=v_t(\psi_t)$ and a single Euler--retraction step starting from a point $\psi_t$ on the manifold:
\[
\psi_{t+\Delta t}^{\mathrm{num}}
=
\mathrm{Retr}_{\psi_t}\big(\Delta t\, v_t(\psi_t)\big),
\qquad
\psi_{t+\Delta t}^{\mathrm{exact}}
=
\mathrm{Exp}_{\psi_t}\big(\Delta t\, v_t(\psi_t)\big)
+ O(\Delta t^2),
\]
the second equality following from the Taylor expansion of the exact solution along the geodesic generated by the exponential map.
Since any first-order retraction agrees with the exponential map up to first order,
\[
\mathrm{Retr}_{\psi_t}\big(\Delta t\, v_t(\psi_t)\big)
=
\mathrm{Exp}_{\psi_t}\big(\Delta t\, v_t(\psi_t)\big)
+ O(\Delta t^2),
\]
so the \emph{local truncation error}, when started from the exact state $\psi_t$, satisfies
\begin{equation}
\label{eq:local_truncation}
\big\|\psi_{t+\Delta t}^{\mathrm{num}}-\psi_{t+\Delta t}^{\mathrm{exact}}\big\|_2
\;\le\; C_1\,\Delta t^2,
\end{equation}
where $C_1$ depends on the second-order regularity of the exact flow on $\mathbb{CP}^{d-1}\times[0,1]$ and is finite by the assumed $C^1$ regularity of $v_t$ and compactness of the manifold.

\smallskip
\noindent\textbf{Step 2: Global first-order convergence on $\mathbb{CP}^{d-1}$.}
Fix a stepsize $\Delta t = 1/K$ for $K\in\mathbb{N}$, let $t_k := k\Delta t$, and let $\Phi_{s,t}$ denote the exact flow of $v_t$ between times $s$ and $t$.
Define the discrete trajectory $\psi_0^{\mathrm{num}} := \psi_0$,
\[
\psi_{k+1}^{\mathrm{num}}
\;:=\;
\mathrm{Retr}_{\psi_{k}^{\mathrm{num}}}\!\big(\Delta t\, v_{t_k}(\psi_{k}^{\mathrm{num}})\big),
\]
and the global error in the unit-sphere ambient norm
\[
E_k \;:=\; \big\|\psi_k^{\mathrm{num}} - \Phi_{0,t_k}(\psi_0)\big\|_2,
\qquad E_0=0.
\]
Assume $v_t$ is $L$-Lipschitz on $\mathbb{CP}^{d-1}$ uniformly in $t\in[0,1]$ in this representation; this is the standard regularity hypothesis behind global convergence of one-step manifold integrators \citep{lee2012smooth}.
Decompose
\[
\psi_{k+1}^{\mathrm{num}} - \Phi_{0,t_{k+1}}(\psi_0)
\;=\;
\underbrace{\big[\psi_{k+1}^{\mathrm{num}} - \Phi_{t_k,t_{k+1}}(\psi_k^{\mathrm{num}})\big]}_{\text{local truncation from }\psi_k^{\mathrm{num}}}
+
\underbrace{\big[\Phi_{t_k,t_{k+1}}(\psi_k^{\mathrm{num}}) - \Phi_{t_k,t_{k+1}}(\Phi_{0,t_k}\psi_0)\big]}_{\text{exact-flow propagation of }E_k}.
\]
By Step~1 applied with starting point $\psi_k^{\mathrm{num}}$, the local-truncation term is bounded by $C_1\,\Delta t^2$.
By the standard Gr\"onwall bound for the exact flow of an $L$-Lipschitz field,
\[
\big\|\Phi_{t_k,t_{k+1}}(\psi)-\Phi_{t_k,t_{k+1}}(\phi)\big\|_2
\;\le\;
e^{L\Delta t}\,\|\psi-\phi\|_2
\quad\text{for all }\psi,\phi\in\mathbb{CP}^{d-1}.
\]
Combining,
\begin{equation}
\label{eq:euler_retraction_recursion}
E_{k+1} \;\le\; e^{L\Delta t}\,E_k + C_1\,\Delta t^2,
\qquad k=0,1,\ldots,K-1.
\end{equation}
Iterating \eqref{eq:euler_retraction_recursion} from $E_0=0$ and using $K\Delta t = 1$,
\[
E_K
\;\le\;
C_1\,\Delta t^2\sum_{j=0}^{K-1} e^{jL\Delta t}
\;=\;
C_1\,\Delta t^2\cdot\frac{e^{KL\Delta t}-1}{e^{L\Delta t}-1}
\;\le\;
C_1\,\Delta t\cdot\frac{e^{L}-1}{L},
\]
where the last inequality uses $e^{L\Delta t}-1\ge L\Delta t$.
The same bound applies to every intermediate index $k\le K$, so
\begin{equation}
\label{eq:global_first_order}
\sup_{0\le k\le K}\big\|\psi_k^{\mathrm{num}} - \Phi_{0,t_k}(\psi_0)\big\|_2
\;\le\;
C_1\,\frac{e^{L}-1}{L}\,\Delta t
\;=\;
O(\Delta t).
\end{equation}
Since the Fubini--Study geodesic distance on $\mathbb{CP}^{d-1}$ is dominated by the ambient chord distance \citep{bengtsson2017geometry}, $d_{\mathrm{FS}}([\psi_k^{\mathrm{num}}],[\Phi_{0,t_k}\psi_0])\le E_k$, and the same $O(\Delta t)$ rate transfers to the intrinsic distance.
This establishes global first-order convergence of the Euler--retraction scheme on the projective manifold; the constant in \eqref{eq:global_first_order} depends only on $L$ and the curvature term $C_1$, both of which are finite under the standing regularity assumptions, and its dimension dependence is inherited from $L$ as discussed in Remark~\ref{rem:lipschitz_d}.
\end{proof}

\section{Additional Experimental Details}
\label{app:exp_overview}

This appendix gathers the full experimental configuration used for all results in Section~\ref{sec:experiments}.
The main text emphasizes the empirical story; here we make the implementation choices explicit by describing the benchmark construction rules, optimization protocol, network architecture, sampling configuration, and baseline-specific hyperparameters.

\subsection{Common training and evaluation protocol}
\label{app:exp_protocol}

Unless otherwise stated, all main-text results are produced by averaging over $10$ independent random seeds under the following shared protocol:
\begin{itemize}
    \item random seeds: $10$ independent runs per configuration, with seed indices $\{1,2,\ldots,10\}$ used to initialize the network parameters, the data sampler, and the base prior;
    \item optimizer: AdamW;
    \item learning rate: $2\times 10^{-4}$;
    \item training batch size: $64$;
    \item evaluation frequency: every $200$ optimization steps;
    \item evaluation batch size: $256$;
    \item training length: $2000$ optimization steps;
    \item reporting rule: for each seed, we select a single checkpoint by validation MMD and report all metrics at that checkpoint;
    \item aggregation: main-text tables report mean $\pm$ standard deviation across the $10$ seeds, computed as the unbiased sample standard deviation, in the format $(\mu \pm \sigma) \times 10^{c}$.
\end{itemize}
For IFM and Euclidean FM, the reported runs use $K=200$ deterministic sampling steps.
All synthetic and physics-inspired ensembles use the same perturbation scale $\epsilon_{\mathrm{cluster}}=0.06$ and, unless otherwise stated, the same local base-prior scale $0.05$.

\paragraph{Fairness of baseline comparisons.}
The IFM, Euclidean FM, SSDM, and Euclidean VP-SDE comparisons use matched network width, depth, training length, evaluation frequency, and $K=200$ sampling or teacher-discretization steps whenever the method admits such a direct match.
QDDPM and QGAN are included as representative circuit-based anchors on the $6$-qubit single-cluster task using their aligned training scripts and reported wall-clock times.
Accordingly, we interpret the circuit-baseline comparison as a useful scale and wall-clock reference, while the central controlled comparison in the paper is IFM versus Euclidean FM, with SSDM and Euclidean VP-SDE serving as diffusion anchors under the matched protocol described above.

\paragraph{Evaluation metrics.}
MMD is computed from the pure-state overlap kernel induced by squared fidelities \citep{gretton2012mmd}.
The observable discrepancy $\Delta_{\mathrm{obs}}$ is the mean absolute difference in the expectations of all one-site $Z_i$ operators and nearest-neighbor $Z_i Z_{i+1}$ correlators.
Ent.~$W_1$ is the empirical Wasserstein-$1$ distance between half-chain entanglement-entropy distributions.
These metrics are computed in the same way for IFM, Euclidean FM, SSDM, and Euclidean VP-SDE, so differences reflect the learned generative distributions rather than method-specific evaluation rules \citep{xu2026ssdm,song2021scorebased}.

\begin{figure*}[t]
\centering
\includegraphics[width=0.94\textwidth]{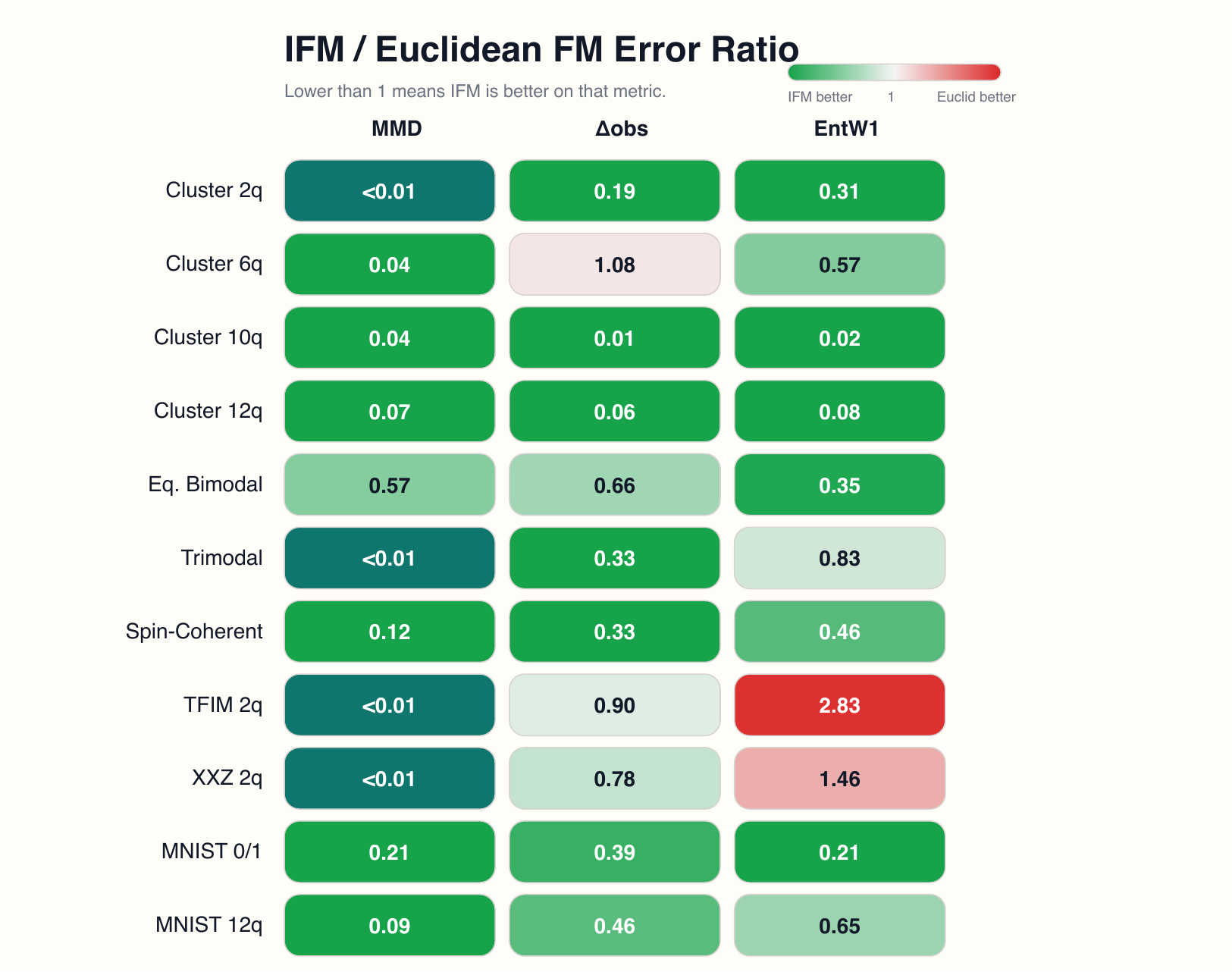}
\caption{Ratio of IFM to Euclidean FM across all reported benchmarks and metrics, computed from the tabled values in this paper. Values below $1$ indicate that IFM attains a lower error. The clearest gains appear on higher-qubit single-cluster tasks, coherence-sensitive multimodal benchmarks, spin-coherent peaks, and the amplitude-encoded MNIST benchmark, while TFIM and XXZ exhibit a more mixed trade-off.}
\label{fig:experiment_ratio_heatmap}
\end{figure*}

\subsection{Benchmark construction}
\label{app:benchmark_construction}

Table~\ref{tab:appendix_benchmark_summary} summarizes the benchmark families used in the paper.
For the synthetic and physics-inspired benchmarks, each target sample is obtained by selecting a canonical center state (or one of several center states), adding small complex Gaussian perturbations in amplitude space, and renormalizing, producing local pure-state ensembles on the projective manifold \citep{bengtsson2017geometry}.

\begin{table*}[t]
\centering
\scriptsize
\caption{Benchmark families used in the experiments. ``Local'' base denotes a small complex Gaussian perturbation around $\ket{0}^{\otimes n}$ followed by renormalization.}
\label{tab:appendix_benchmark_summary}
\begin{tabular}{p{2.1cm}p{0.9cm}p{1.5cm}p{7.5cm}}
\toprule
Benchmark & Qubits & Base prior & Construction \\
\midrule
Single-cluster & $2,6,10,12$ in main text; $14$ stress test & Local & Perturb-and-renormalize ensemble around $\ket{0}^{\otimes n}$ with amplitude noise scale $0.06$. \\
Equatorial bimodal & $2,10$ in main text; $6$ in appendix; $14$ stress test & Local & Equal mixture of $(\ket{0}^{\otimes n}+\ket{1}^{\otimes n})/\sqrt{2}$ and $(\ket{0}^{\otimes n}-\ket{1}^{\otimes n})/\sqrt{2}$. \\
Trimodal & $2,10$ in main text; $6$ in appendix; $14$ stress test & Local & Equal mixture of $\ket{0}^{\otimes n}$, $(\ket{0}^{\otimes n}+\ket{1}^{\otimes n})/\sqrt{2}$, and $(\ket{0}^{\otimes n}-\ket{1}^{\otimes n})/\sqrt{2}$. \\
Spin-coherent peaks & $2,10$ in main text; $6$ in appendix; $14$ stress test & Local & Two product spin-coherent states built from single-qubit coherent orientations along $+x$ and $+y$. \\
TFIM family & $2,10$ in main text; $6$ in appendix & Local & Exact ground states of $H_{\mathrm{TFIM}}=-\sum_{i=1}^{n-1}Z_iZ_{i+1}-g\sum_{i=1}^n X_i$ with $g\in\{0.2,0.5,1.0,2.0\}$, then perturbed and renormalized. \\
XXZ family & $2,10$ in main text; $6$ in appendix & Local & Exact ground states of $H_{\mathrm{XXZ}}=\sum_{i=1}^{n-1}(X_iX_{i+1}+Y_iY_{i+1}+\Delta Z_iZ_{i+1})$ with $\Delta\in\{-1.0,0.0,0.5,1.0\}$, then perturbed and renormalized. \\
MNIST $0/1$ & $6$ in main text; $12$ extension; $14$ stress test & Local & Balanced MNIST digits $\{0,1\}$, flattened, reduced to $\min(2^n,784)$ real dimensions, centered per example, zero-padded when needed, and normalized as real amplitude vectors. \\
\bottomrule
\end{tabular}
\end{table*}

\paragraph{Single-cluster benchmark.}
For $n$ qubits, the canonical center is $\ket{0}^{\otimes n}$.
Each sample is generated by adding i.i.d.\ complex Gaussian noise of standard deviation $\epsilon_{\mathrm{cluster}}$ to this center vector and renormalizing.
This is the easiest unimodal benchmark and is used primarily to probe scaling with qubit number.

\paragraph{Structured multimodal benchmarks.}
The equatorial bimodal benchmark uses the two GHZ-type equatorial centers \citep{bengtsson2017geometry,preskill2018nisq}
\[
\frac{\ket{0}^{\otimes n}+\ket{1}^{\otimes n}}{\sqrt{2}},
\qquad
\frac{\ket{0}^{\otimes n}-\ket{1}^{\otimes n}}{\sqrt{2}}.
\]
The trimodal benchmark augments these two coherent equatorial modes with the pole state $\ket{0}^{\otimes n}$.
These tasks are deliberately designed to be phase-sensitive and therefore test whether intrinsic projective transport is more informative than ambient Euclidean interpolation.

\paragraph{Spin-coherent benchmark.}
We use the standard single-qubit coherent-state parameterization \citep{arecchi1972atomic}
\[
\ket{\theta,\phi}
=
\cos(\theta/2)\ket{0}
{}
e^{i\phi}\sin(\theta/2)\ket{1},
\]
with the two orientations $(\theta,\phi)=(\pi/2,0)$ and $(\pi/2,\pi/2)$, corresponding to the $+x$ and $+y$ directions on the Bloch sphere.
The $n$-qubit benchmark centers are the corresponding tensor-product states.

\paragraph{TFIM and XXZ families.}
For the TFIM and XXZ benchmarks, we use compact spin-chain families inspired by standard condensed-matter models \citep{chaikin1995principles,bose2003quantum}.
For the TFIM benchmark we diagonalize
\[
H_{\mathrm{TFIM}}=-\sum_{i=1}^{n-1}Z_iZ_{i+1}-g\sum_{i=1}^{n}X_i,
\qquad
g\in\{0.2,0.5,1.0,2.0\},
\]
and for the XXZ benchmark we diagonalize
\[
H_{\mathrm{XXZ}}=\sum_{i=1}^{n-1}(X_iX_{i+1}+Y_iY_{i+1}+\Delta Z_iZ_{i+1}),
\qquad
\Delta\in\{-1.0,0.0,0.5,1.0\}.
\]
The resulting exact ground states are then used as benchmark centers before perturbation and renormalization.
These families are evaluated at $2$ qubits in the main text and extended to $6$ qubits in the appendix.
These tasks are intended as compact proof-of-concept physics-inspired families rather than full many-body benchmarks.
They should therefore be read as controlled sanity checks for whether IFM behaves sensibly on Hamiltonian-generated state families, not as broad evidence for many-body physics performance.
In particular, the TFIM and XXZ rows use only four Hamiltonian parameters each, so the main empirical support for the geometry claim comes from the controlled IFM-versus-Euclidean-FM comparisons across scaling, multimodal, phase-sensitive, and appendix stress-test settings rather than from these compact families alone.

\paragraph{MNIST-derived amplitude-vector benchmark.}
For MNIST we use balanced digits $\{0,1\}$ from the training set \citep{lecun2002gradient}.
This benchmark is a classical image-vector distribution embedded into the pure-state representation, rather than experimentally collected quantum data.
For the $6$-qubit experiment, images are flattened, rescaled to $[0,1]$, reduced to $64$ dimensions, centered by subtracting the per-example mean, and normalized to unit Euclidean norm.
For the $12$-qubit extension, we use the flattened $784$-dimensional vectors, center them per example, and zero-pad them to dimension $2^{12}=4096$ before normalization.
In both cases, the resulting real vectors are interpreted as pure-state amplitudes.

\begin{figure*}[t]
\centering
\includegraphics[width=0.97\textwidth]{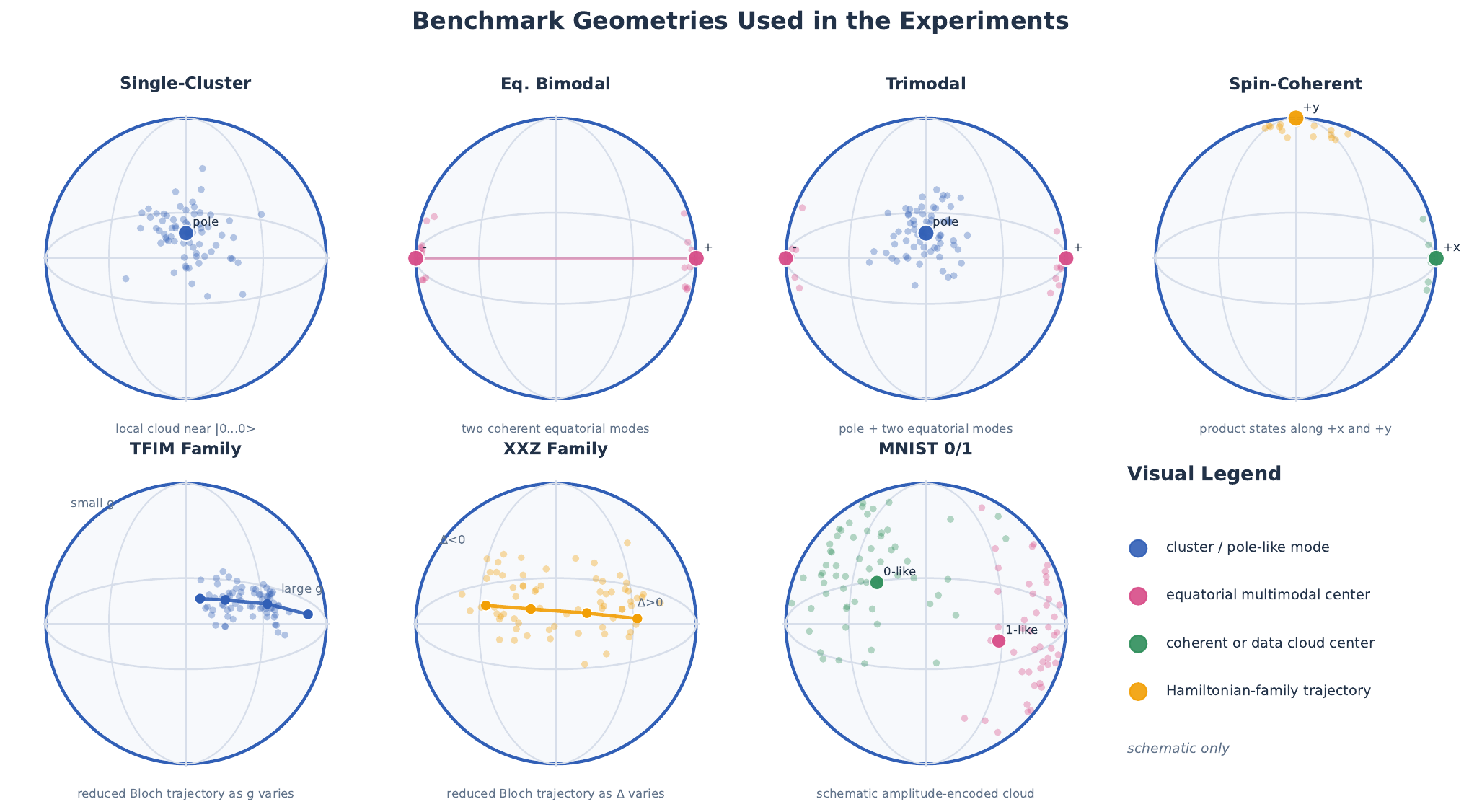}
\caption{Effective Bloch-sphere views of the benchmark families used in our experiments. The panels are schematic and are included to make the qualitative geometry behind Table~\ref{tab:structured_benchmarks} and Table~\ref{tab:appendix_6q_structured} visually interpretable.}
\label{fig:benchmark_bloch_schematics}
\end{figure*}

\paragraph{How to read Figure~\ref{fig:benchmark_bloch_schematics}.}
The synthetic multimodal benchmarks are shown by their mode centers on an effective two-level sphere generated by $\ket{0}^{\otimes n}$ and $\ket{1}^{\otimes n}$.
The spin-coherent benchmark uses physical single-qubit coherent orientations, while the TFIM and XXZ panels show reduced Bloch trajectories of their ground-state families as the Hamiltonian parameter varies.
The MNIST panel is an effective cloud sketch after amplitude encoding and should be read only as a geometric summary rather than as an exact Bloch embedding of the full data distribution.

\subsection{Model architectures}
\label{app:model_architectures}

All learned baselines use the same basic MLP scale so that the main empirical differences are due to geometry and training objective rather than network capacity.
Every model uses a $128$-dimensional sinusoidal time embedding followed by a small time-embedding MLP and a main backbone of width $512$ and depth $5$.

\paragraph{IFM.}
The IFM model in \texttt{sse1\_fm.py} predicts a complex velocity field from the concatenation of the real and imaginary parts of the current state together with the time embedding, following the velocity-field parameterization used in flow matching while enforcing projective tangency \citep{lipman2023flowmatchinggenerativemodeling,chen2023flow}.
Its output is projected to the tangent space at the current state.
In the main experiments, the conditional path family is the phase-aligned normalized-chord path of Section~\ref{sec:conditional-paths}, i.e.\ \texttt{path\_type=chord}.

\paragraph{Euclidean FM.}
The Euclidean FM baseline in \texttt{sse1\_fm\_euclid.py} uses the same time-embedding size, width, and depth, but acts on the ambient real representation in $\mathbb{R}^{2d}$ \citep{lipman2023flowmatchinggenerativemodeling,lipman2022flow}.
It predicts an unconstrained Euclidean velocity and does not quotient out global phase or impose horizontal tangency.

\paragraph{SSDM.}
The intrinsic diffusion baseline in \texttt{sse1.py} uses a score network of the same width ($512$), depth ($5$), and time-embedding dimension ($128$) \citep{xu2026ssdm}.
Its output is interpreted as a complex score and projected back to the tangent space, consistent with the local OU/VP teacher construction used in that model.

\paragraph{Euclidean VP-SDE.}
The Euclidean diffusion baseline in \texttt{esde.py} also uses a $128$-dimensional time embedding and a width-$512$, depth-$5$ MLP, but predicts an ambient score in the real concatenated state-vector representation \citep{song2021scorebased}.

\subsection{Code and supplementary materials}
\label{app:code_contribution}

The complete implementation used for the experiments is provided in the supplementary material accompanying the submission.
In particular, the supplementary package contains the key Python entry points used throughout the empirical study:
\begin{itemize}
    \item \texttt{sse1\_fm.py}: the main IFM implementation, including the phase-aligned normalized-chord path, the geodesic ablation, and the appendix-only generic-geodesic baseline;
    \item \texttt{sse1\_fm\_euclid.py}: the Euclidean flow-matching baseline in the ambient real representation;
    \item \texttt{sse1.py}: the SSDM baseline;
    \item \texttt{esde.py}: the Euclidean VP-SDE baseline;
    \item the QDDPM and QGAN baseline scripts used in the representative diffusion and adversarial comparisons reported in Table~\ref{tab:single_cluster_baselines}.
\end{itemize}
The main IFM algorithmic logic is concentrated in \texttt{sse1\_fm.py}.
That file implements the data generators, conditional-path constructors, tangent-projected velocity network, ODE sampling routine, and evaluation metrics.
We include it in the supplement so that the main empirical claims can be reproduced directly from a single reference implementation.

\subsection{Baseline-specific hyperparameters}
\label{app:baseline_configs}

Table~\ref{tab:appendix_baseline_configs} collects the method-specific settings used in the paper.
For Euclidean VP-SDE, QDDPM, and QGAN, which are only evaluated on the representative $6$-qubit single-cluster benchmark, we report the exact settings used in those runs \citep{song2021scorebased,zhang2024generative,lloyd2018quantum,dallaire2018quantum}.
For SSDM, we use the same underlying implementation family both on the representative single-cluster comparison and on the $2$- and $10$-qubit structured benchmarks, with the dataset mode changed accordingly \citep{xu2026ssdm}.

\begin{table*}[t]
\centering
\scriptsize
\caption{Baseline-specific hyperparameters for the reported experiments. The first two rows summarize the settings used throughout the main IFM-versus-Euclidean-FM comparison \citep{lipman2023flowmatchinggenerativemodeling}, while the remaining rows summarize SSDM \citep{xu2026ssdm}, Euclidean VP-SDE \citep{song2021scorebased}, QDDPM \citep{zhang2024generative}, and QGAN \citep{lloyd2018quantum,dallaire2018quantum} on the representative $6$-qubit single-cluster comparison reported in Table~\ref{tab:single_cluster_baselines}.}
\label{tab:appendix_baseline_configs}
\begin{tabular}{p{2.3cm}p{10.9cm}}
\toprule
Method & Configuration details \\
\midrule
IFM (ours) &
AdamW, learning rate $2\times 10^{-4}$, batch size $64$, training steps $2000$, evaluation every $200$ steps, evaluation batch $256$, deterministic sampling with $K=200$ and $\Delta t=1/200$, data-pool size $4096$, base prior \texttt{local} with scale $0.05$, and \texttt{path\_type=chord}. \\
\addlinespace
Geodesic IFM (ours variant) &
Representative $6$-qubit single-cluster ablation only. Same optimizer, network scale, data pool, batch size, sample budget, and deterministic sampler as IFM, but with \texttt{path\_type=geodesic}, i.e.\ a phase-aligned Fubini--Study-geodesic interpolation used in place of the normalized-chord surrogate. \\
\addlinespace
Generic geodesic FM~\citep{chen2023flow} &
Representative $6$-qubit single-cluster appendix baseline only. Same optimizer, network scale, data pool, batch size, sample budget, and deterministic sampler as IFM, but with \texttt{path\_type=generic\_geodesic}. This variant omits Pancharatnam phase alignment and therefore excludes the quotient-specific endpoint correction introduced by our method. \\
\addlinespace
Euclidean FM &
AdamW, learning rate $2\times 10^{-4}$, batch size $64$, training steps $2000$, evaluation every $200$ steps, evaluation batch $256$, deterministic sampling with $K=200$ in the reported runs, data-pool size $4096$, and local ambient Gaussian base with scale $0.05$. \\
\addlinespace
SSDM &
Used for the representative $6$-qubit single-cluster comparison and for the $2$- and $10$-qubit structured benchmarks in Table~\ref{tab:structured_benchmarks}. AdamW, learning rate $2\times 10^{-4}$, batch size $64$, training steps $2000$, evaluation every $200$ steps, local OU/VP teacher with $K=200$, $\Delta t=\delta=1/200$, $\sigma_{\min}=0.05$, $\sigma_{\max}=1.0$, mean-reversion parameter $\lambda_0=0.2$, and the same benchmark-specific state generators used in the flow-matching scripts via \texttt{SSDM\_DATA\_MODE}. \\
\addlinespace
Euclidean VP-SDE &
Representative $6$-qubit single-cluster run only. AdamW, learning rate $2\times 10^{-4}$, batch size $64$, training steps $2000$, evaluation every $200$ steps, evaluation batch $256$, reverse VP-SDE sampling with $K=200$, and linear schedule parameters $\beta_{\min}=0.1$ and $\beta_{\max}=20.0$. \\
\addlinespace
QDDPM &
Representative $6$-qubit single-cluster run only. Executed in the \texttt{qddpm} conda environment using \texttt{QDDPM\_torch.py}. The baseline uses time-dependent PQCs with $6$ layers per diffusion step over $T=20$ steps. The wall-clock training time measured in our run is reported in Table~\ref{tab:appendix_baseline_timing}. Metrics are computed with the same definitions as in the other baselines. \\
\addlinespace
QGAN &
Representative $6$-qubit single-cluster run only. Executed in the \texttt{qddpm} conda environment using \texttt{QGAN.py}, \texttt{qgan\_train\_1.py}, and the aligned runner \texttt{qgan-1.py}. The baseline uses a $120$-layer PQC generator together with a $16$-layer PQC discriminator. The wall-clock training time measured in our run is reported in Table~\ref{tab:appendix_baseline_timing}. \\
\bottomrule
\end{tabular}
\end{table*}

\begin{table}[t]
\centering
\small
\caption{Training-time comparison on the representative $6$-qubit single-cluster benchmark, including QDDPM \citep{zhang2024generative}, QGAN \citep{lloyd2018quantum,dallaire2018quantum}, and our method. Lower is better. Speedup is measured relative to QDDPM.}
\label{tab:appendix_baseline_timing}
\begin{tabular}{lcc}
\toprule
Method & Train time (s) $\downarrow$ & Speedup \\
\midrule
QDDPM & 21150.52 & $1.0\times$ \\
QGAN & 198.64 & $106.4\times$ \\
IFM (ours, chord) & 240.91 & $87.8\times$ \\
\bottomrule
\end{tabular}
\end{table}

\begin{table}[t]
\centering
\small
\caption{Appendix-only $10$-qubit single-cluster extension of Table~\ref{tab:single_cluster_baselines} for the scalable baselines: Euclidean VP-SDE \citep{song2021scorebased}, SSDM \citep{xu2026ssdm}, and Euclidean FM \citep{lipman2023flowmatchinggenerativemodeling}. Lower is better. QDDPM and QGAN are omitted here because our reported circuit-based runs were only executed on the representative $6$-qubit benchmark.}
\label{tab:appendix_single_cluster_10q}
\begin{tabular}{lccc}
\toprule
Method & MMD $\downarrow$ & $\Delta_{\mathrm{obs}} \downarrow$ & Ent.~$W_1 \downarrow$ \\
\midrule
Euclidean VP-SDE & $2.3049 \times 10^{-2}$ & $1.1782 \times 10^{-1}$ & $8.7980 \times 10^{-2}$ \\
SSDM & $1.4157 \times 10^{-2}$ & $1.1883 \times 10^{-1}$ & $8.6077 \times 10^{-2}$ \\
Euclidean FM & $1.6830 \times 10^{-2}$ & $9.5109 \times 10^{-2}$ & $8.2071 \times 10^{-2}$ \\
Geodesic IFM (ours variant) & $7.3038 \times 10^{-4}$ & $1.3659 \times 10^{-3}$ & $1.5996 \times 10^{-3}$ \\
IFM (ours) & $7.3779 \times 10^{-4}$ & $1.3936 \times 10^{-3}$ & $1.4582 \times 10^{-3}$ \\
\bottomrule
\end{tabular}
\end{table}

\paragraph{Scope of the diffusion comparisons.}
Our main experimental question is whether respecting the intrinsic geometry of $\mathbb{CP}^{d-1}$ matters once the learning principle is held fixed.
That question is most directly isolated by comparing IFM to Euclidean FM across the full benchmark suite.
We therefore use Euclidean VP-SDE as a representative ambient diffusion anchor only on the single-cluster task, while SSDM is evaluated both on that representative task and on the small- and higher-qubit structured benchmarks where an intrinsic diffusion comparison remains especially informative \citep{song2021scorebased,xu2026ssdm}.
Even so, the broader sweep still concentrates on the sharpest experimental contrast in the paper, namely intrinsic-versus-ambient flow matching.
Table~\ref{tab:appendix_single_cluster_10q} extends the single-cluster baseline picture to $10$ qubits for the scalable baselines and shows that the separation between IFM and the diffusion baselines becomes much larger once the dense-state dimension increases.

\subsection{Sampling and reporting details}
\label{app:reporting_details}

For IFM, samples are generated by integrating the learned tangent field from the chosen base prior and applying the normalization-based retraction at each step.
For Euclidean FM, samples are generated in the ambient real representation and then mapped back to valid pure states by complex reconstruction and final normalization.
For SSDM and Euclidean VP-SDE, we use the reverse-time stochastic samplers implemented in their respective scripts with the settings summarized above.

All numbers reported in the main text are produced under the matched training-budget and evaluation protocol described above.
More precisely, for each seed we evaluate the model every $200$ optimization steps, select the checkpoint with the best validation MMD, and then report MMD, $\Delta_{\mathrm{obs}}$, and Ent.~$W_1$ together at that single checkpoint.
The main-text table entries report the mean and the unbiased sample standard deviation across the $10$ seeds (seed indices $\{1,\ldots,10\}$); the observed variances are small relative to the mean differences between methods and do not affect the qualitative conclusions.
Because the main protocol is intentionally lightweight, Appendix~\ref{app:14q_flow_stress} additionally includes a $20{,}000$-step high-dimensional flow-matching check to test whether the Euclidean FM gap on a $14$-qubit task is mainly an undertraining artifact.
Our aim in this paper is to establish the empirical value of the geometric inductive bias rather than to present a final large-scale benchmark study.
Whenever a value is shown as ``$\approx 0$'', the underlying number was positive but rounded to zero at the output precision used during training.

\subsection{Additional single-cluster discussion}
\label{app:single_cluster_discussion}

The single-cluster benchmark is intentionally the easiest setting in the paper, but it is useful for isolating scaling effects.
At $2$ qubits, both IFM and Euclidean FM fit the target well, while IFM still improves all three metrics.
At $6$ qubits, the main improvement is in MMD and entanglement structure, with comparable observable discrepancy.
By $10$ and $12$ qubits, the gap becomes much larger across all reported metrics, consistent with the view that projective geometric bias becomes more valuable as ambient redundancy grows.
The geodesic IFM ablation is close to IFM on the representative $6$-qubit single-cluster task, suggesting that for this easy unimodal family the main benefit is intrinsic transport itself rather than the precise choice between geodesic and phase-aligned normalized-chord paths.
Table~\ref{tab:appendix_single_cluster_10q} extends the scalable-baseline comparison to $10$ qubits, and Table~\ref{tab:appendix_baseline_timing} reports the large training-time gap between circuit baselines and IFM.

\subsection{Fourteen-qubit flow-matching stress test}
\label{app:14q_flow_stress}

Table~\ref{tab:appendix_14q_flow_stress} reports an additional $14$-qubit stress test comparing only IFM and Euclidean FM on the scalable benchmark families.
These runs use the same $2000$ training steps, $K=200$ deterministic sampling steps, network scale, and evaluation protocol as the main flow-matching experiments.
We include them to show the behavior of the intrinsic-versus-ambient flow comparison at a larger dense-state dimension.

\begin{table}[t]
\centering
\small
\caption{$14$-qubit stress test for IFM and Euclidean FM on scalable benchmark families. Lower is better.}
\label{tab:appendix_14q_flow_stress}
\begin{tabular}{llccc}
\toprule
Benchmark & Method & MMD $\downarrow$ & $\Delta_{\mathrm{obs}} \downarrow$ & Ent.~$W_1 \downarrow$ \\
\midrule
Single-cluster & Euclidean FM & $7.8806 \times 10^{-3}$ & $7.9683 \times 10^{-3}$ & $2.0976 \times 10^{-3}$ \\
Single-cluster & IFM (ours) & $2.7314 \times 10^{-4}$ & $3.2485 \times 10^{-3}$ & $1.5192 \times 10^{-3}$ \\
\addlinespace
Equatorial bimodal & Euclidean FM & $7.8461 \times 10^{-3}$ & $4.1493 \times 10^{-3}$ & $1.2981 \times 10^{-3}$ \\
Equatorial bimodal & IFM (ours) & $3.2034 \times 10^{-4}$ & $4.0576 \times 10^{-3}$ & $1.0721 \times 10^{-3}$ \\
\addlinespace
Trimodal & Euclidean FM & $7.8489 \times 10^{-3}$ & $5.0208 \times 10^{-3}$ & $1.5435 \times 10^{-3}$ \\
Trimodal & IFM (ours) & $3.2284 \times 10^{-4}$ & $2.7896 \times 10^{-3}$ & $1.2578 \times 10^{-3}$ \\
\addlinespace
Spin-coherent & Euclidean FM & $7.8462 \times 10^{-3}$ & $4.0279 \times 10^{-4}$ & $2.4284 \times 10^{-3}$ \\
Spin-coherent & IFM (ours) & $2.8126 \times 10^{-4}$ & $5.8993 \times 10^{-4}$ & $2.3783 \times 10^{-3}$ \\
\addlinespace
MNIST $0/1$ & Euclidean FM & $1.5637 \times 10^{-1}$ & $3.2145 \times 10^{-1}$ & $2.8245 \times 10^{0}$ \\
MNIST $0/1$ & IFM (ours) & $2.4333 \times 10^{-2}$ & $1.4485 \times 10^{-1}$ & $1.7042 \times 10^{0}$ \\
\bottomrule
\end{tabular}
\end{table}

To check whether the $2000$-step budget is simply undertraining the ambient baseline in the high-dimensional regime, we additionally ran a $20{,}000$-step comparison on the $14$-qubit single-cluster task.
Figure~\ref{fig:longrun_14q_cluster} plots MMD, $\Delta_{\mathrm{obs}}$, and Ent.~$W_1$ throughout training.
Euclidean FM remains essentially flat in MMD over the longer run, ending at $7.88\times 10^{-3}$, close to its $2000$-step value in Table~\ref{tab:appendix_14q_flow_stress}.
IFM remains more than an order of magnitude lower in MMD and also retains lower observable and entanglement discrepancies by the end of the run.
This does not rule out all possible effects of hyperparameter tuning or architecture changes, but it indicates that the reported high-dimensional gap is not explained by the short training schedule alone.

\begin{figure*}[t]
\centering
\includegraphics[width=0.92\textwidth]{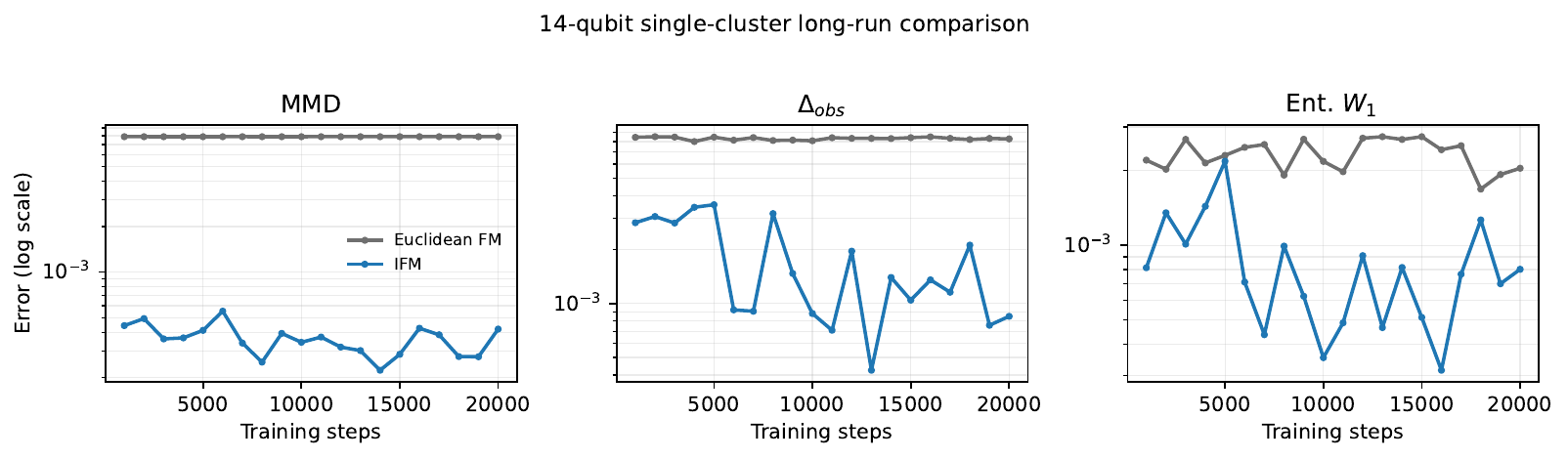}
\caption{$20{,}000$-step high-dimensional training check on the $14$-qubit single-cluster benchmark. Both methods use the same network scale, batch size, local base prior, evaluation batch size, deterministic sampler with $K=200$, and optimizer settings as the main flow-matching experiments, except for the longer training horizon and evaluation every $1000$ steps. Lower is better.}
\label{fig:longrun_14q_cluster}
\end{figure*}

\subsection{Additional structured-benchmark discussion}
\label{app:structured_discussion}

The synthetic multimodal tasks highlight where phase-aware projective transport matters most.
In the equatorial bimodal benchmark, the two modes differ by relative phase rather than by occupation of distant computational-basis corners; the trimodal task adds a pole state to these coherent modes.
These are precisely the settings where Pancharatnam-aligned conditional paths are more informative than ambient Euclidean interpolation.
The SSDM rows show that intrinsic diffusion can already improve over Euclidean FM on several structured tasks, but IFM often gives stronger MMD and better coherence-sensitive behavior, especially at $10$ qubits.
The spin-coherent benchmark gives a related geometric test using product states oriented along different Bloch-sphere directions; IFM again achieves much smaller MMD and entanglement Wasserstein distance at larger scale.

\subsection{Six-qubit extensions of the structured benchmarks}
\label{app:6q_structured}

To test whether the trends in Table~\ref{tab:structured_benchmarks} remain stable between the small-system and higher-qubit regimes, we also evaluated $6$-qubit versions of the first five benchmark families from that table.
Table~\ref{tab:appendix_6q_structured} reports the corresponding $6$-qubit results under the same matched protocol.
The intermediate-scale picture is deliberately more nuanced than a uniform-win narrative: IFM remains clearly stronger on trimodal and TFIM-style structured transport, stays competitive on the equatorial bimodal benchmark, improves MMD and observable matching on the spin-coherent benchmark, and exhibits a more mixed trade-off on XXZ.

\begin{table*}[t]
\centering
\small
\caption{Six-qubit extensions of the first five benchmark families from Table~\ref{tab:structured_benchmarks}, comparing IFM with Euclidean FM \citep{lipman2023flowmatchinggenerativemodeling}. Lower is better. These appendix-only results probe the intermediate regime between the small-system and $10$-qubit results reported in the main text.}
\label{tab:appendix_6q_structured}
\begin{tabular}{llccc}
\toprule
Benchmark & Method & MMD $\downarrow$ & $\Delta_{\mathrm{obs}} \downarrow$ & Ent.~$W_1 \downarrow$ \\
\midrule
Eq. bimodal ($6$ qubits) & Euclidean FM & $2.1861 \times 10^{-2}$ & $1.0823 \times 10^{-1}$ & $2.4905 \times 10^{-2}$ \\
Eq. bimodal ($6$ qubits) & IFM (ours) & $6.2495 \times 10^{-3}$ & $6.4122 \times 10^{-2}$ & $2.9000 \times 10^{-2}$ \\
\addlinespace
Trimodal ($6$ qubits) & Euclidean FM & $1.3397 \times 10^{-2}$ & $4.2138 \times 10^{-2}$ & $6.5884 \times 10^{-2}$ \\
Trimodal ($6$ qubits) & IFM (ours) & $3.6621 \times 10^{-4}$ & $2.0026 \times 10^{-2}$ & $5.3195 \times 10^{-2}$ \\
\addlinespace
Spin-coherent ($6$ qubits) & Euclidean FM & $1.6100 \times 10^{-2}$ & $1.1463 \times 10^{-2}$ & $1.5728 \times 10^{-1}$ \\
Spin-coherent ($6$ qubits) & IFM (ours) & $4.7601 \times 10^{-3}$ & $6.5775 \times 10^{-3}$ & $2.0694 \times 10^{-1}$ \\
\addlinespace
TFIM ($6$ qubits) & Euclidean FM & $3.1681 \times 10^{-2}$ & $5.6582 \times 10^{-2}$ & $6.2842 \times 10^{-2}$ \\
TFIM ($6$ qubits) & IFM (ours) & $1.7825 \times 10^{-3}$ & $1.0956 \times 10^{-2}$ & $2.5062 \times 10^{-2}$ \\
\addlinespace
XXZ ($6$ qubits) & Euclidean FM & $1.1880 \times 10^{-2}$ & $1.0420 \times 10^{-2}$ & $2.7554 \times 10^{-2}$ \\
XXZ ($6$ qubits) & IFM (ours) & $1.2254 \times 10^{-2}$ & $4.8443 \times 10^{-2}$ & $9.1078 \times 10^{-2}$ \\
\bottomrule
\end{tabular}
\end{table*}

\subsection{Six-qubit geometry and path-choice ablation summary}
\label{app:geometry_path_ablation_6q}

Table~\ref{tab:appendix_geometry_path_summary} summarizes the metric-wise winners among three flow-matching variants at $6$ qubits:
ambient Euclidean FM, generic geodesic FM on the manifold, and IFM.
The purpose is to disentangle three effects that are otherwise coupled in the main comparison: moving from ambient to intrinsic geometry, choosing a generic manifold geodesic versus a quantum-specific path, and applying Pancharatnam phase alignment.
The summary shows that intrinsic geometry explains a large part of the gain over Euclidean FM, but the remaining advantage of IFM is benchmark dependent.
In particular, IFM is strongest on trimodal and TFIM at $6$ qubits, while generic geodesic FM is the clear winner on XXZ at the same scale.
This supports a more nuanced interpretation: phase-aligned IFM is not uniformly best across all summaries, but it is especially effective when coherent mode organization makes endpoint gauge choice and path coupling important.

\begin{table}[t]
\centering
\small
\caption{Metric-wise winners in the $6$-qubit geometry/path ablation. Candidates are Euclidean FM, generic geodesic FM \citep{chen2023flow}, and IFM. Lower is better for all metrics.}
\label{tab:appendix_geometry_path_summary}
\begin{tabular}{p{2.0cm}p{2.4cm}p{2.3cm}p{2.3cm}p{3.2cm}}
\toprule
Benchmark & MMD winner & $\Delta_{\mathrm{obs}}$ winner & Ent.~$W_1$ winner & Takeaway \\
\midrule
Single-cluster & Generic geodesic FM & Euclidean FM & IFM & Mostly geometry-driven; easy task. \\
Eq. bimodal & IFM & IFM & Euclidean FM & Phase alignment helps MMD/observables. \\
Trimodal & IFM & IFM & IFM & Strongest evidence for IFM coupling. \\
Spin-coherent & IFM & IFM & Euclidean FM & MMD/observables improve; entanglement mixed. \\
TFIM & IFM & IFM & IFM & IFM best among these flow variants. \\
XXZ & Generic geodesic FM & Generic geodesic FM & Generic geodesic FM & Generic geodesic coupling is better here. \\
\bottomrule
\end{tabular}
\end{table}

\subsection{Phase-alignment ablation on six-qubit structured benchmarks}
\label{app:phase_ablation_6q}

To isolate the role of the phase-alignment component itself, we additionally compare IFM against a matched intrinsic baseline in which the conditional path is replaced by the raw generic geodesic interpolation from Appendix~\ref{app:interpolation_details}.
This ablation keeps the same IFM network, optimizer, tangent projection, and deterministic sampling scheme, and changes only one modeling ingredient: the Pancharatnam-aligned endpoint representative is removed.
In other words, the baseline still learns an intrinsic manifold flow, but it no longer fixes the endpoint gauge before interpolation.

Table~\ref{tab:appendix_phase_ablation} reports the resulting $6$-qubit comparison on all five structured benchmarks from Table~\ref{tab:structured_benchmarks}.
On the equatorial bimodal and trimodal tasks, removing phase alignment consistently worsens MMD and observable matching, with the trimodal benchmark showing the clearest across-the-board degradation.
On TFIM, the no-alignment variant is also markedly worse on all three metrics, while on spin-coherent and XXZ the comparison is more mixed, indicating that phase alignment is most important when the target ensemble is organized by coherent multimodal structure rather than by a smoother family geometry alone.
We therefore interpret phase alignment as a specifically useful ingredient for quantum ensembles whose transport depends sensitively on quotient-consistent mode organization.

\begin{table}[t]
\centering
\small
\caption{Ablation of the phase-alignment component on five $6$-qubit structured benchmarks. Lower is better. ``No phase alignment'' uses the same intrinsic FM model with \texttt{path\_type=generic\_geodesic}, i.e.\ the raw endpoint representatives are interpolated directly without Pancharatnam gauge fixing.}
\label{tab:appendix_phase_ablation}
\begin{tabular}{llccc}
\toprule
Benchmark & Method & MMD $\downarrow$ & $\Delta_{\mathrm{obs}} \downarrow$ & Ent.~$W_1 \downarrow$ \\
\midrule
Eq. bimodal ($6$ qubits) & No phase alignment & $8.1550 \times 10^{-3}$ & $6.5664 \times 10^{-2}$ & $2.8592 \times 10^{-2}$ \\
Eq. bimodal ($6$ qubits) & IFM (ours) & $6.2495 \times 10^{-3}$ & $6.4122 \times 10^{-2}$ & $2.9000 \times 10^{-2}$ \\
\addlinespace
Trimodal ($6$ qubits) & No phase alignment & $6.1315 \times 10^{-4}$ & $2.3808 \times 10^{-2}$ & $5.3652 \times 10^{-2}$ \\
Trimodal ($6$ qubits) & IFM (ours) & $3.6621 \times 10^{-4}$ & $2.0026 \times 10^{-2}$ & $5.3195 \times 10^{-2}$ \\
\addlinespace
Spin-coherent ($6$ qubits) & No phase alignment & $1.1341 \times 10^{-2}$ & $9.1691 \times 10^{-3}$ & $1.6435 \times 10^{-1}$ \\
Spin-coherent ($6$ qubits) & IFM (ours) & $4.7601 \times 10^{-3}$ & $6.5775 \times 10^{-3}$ & $2.0694 \times 10^{-1}$ \\
\addlinespace
TFIM ($6$ qubits) & No phase alignment & $1.8038 \times 10^{-2}$ & $5.3284 \times 10^{-2}$ & $1.2640 \times 10^{-1}$ \\
TFIM ($6$ qubits) & IFM (ours) & $1.7825 \times 10^{-3}$ & $1.0956 \times 10^{-2}$ & $2.5062 \times 10^{-2}$ \\
\addlinespace
XXZ ($6$ qubits) & No phase alignment & $9.5237 \times 10^{-3}$ & $9.9533 \times 10^{-3}$ & $1.0297 \times 10^{-2}$ \\
XXZ ($6$ qubits) & IFM (ours) & $1.2254 \times 10^{-2}$ & $4.8443 \times 10^{-2}$ & $9.1078 \times 10^{-2}$ \\
\bottomrule
\end{tabular}
\end{table}

\subsection{Generic manifold-flow-matching baseline on six-qubit single-cluster and structured tasks}
\label{app:fmgg_baseline}

To complement the Euclidean and quantum-specific baselines, we also evaluated an appendix-only generic manifold-flow-matching variant inspired by \emph{Flow Matching on General Geometries} \citep{chen2023flow}.
This is an author-constructed comparator rather than a baseline reported in the original RFM paper, which does not include a $\mathbb{CP}^{d-1}$ quantum pure-state benchmark or a Pancharatnam/Hopf-quotient implementation.
The goal of this comparison is deliberately narrow: we keep the same IFM network, optimizer, tangent projection, sampling step count, deterministic sampler, data pools, and evaluation metrics, but replace our quotient-aware path construction by the raw geodesic interpolation described in Appendix~\ref{app:interpolation_details}.
In particular, this baseline does \emph{not} perform Pancharatnam phase alignment before defining the conditional path.
It is therefore a strong generic intrinsic baseline, but it is not specifically designed for the projective gauge structure of quantum pure states.

Table~\ref{tab:appendix_fmgg_baseline} shows the resulting $6$-qubit comparison on the easy single-cluster benchmark together with the five structured benchmarks from Table~\ref{tab:structured_benchmarks}, plus one higher-qubit XXZ extension.
On the single-cluster task, the generic geodesic manifold baseline already captures most of the gain obtained by moving from ambient Euclidean flow matching to an intrinsic manifold formulation.
Across the structured benchmarks, its behavior is more mixed: it improves substantially over Euclidean FM on equatorial bimodal, trimodal, TFIM, and XXZ in MMD, but remains clearly weaker than IFM on TFIM and on the coherence-sensitive multimodal tasks in overall balance.
XXZ at $6$ qubits is the main counterexample, where the generic geodesic manifold baseline is the strongest of the three.
By contrast, on the $10$-qubit XXZ extension, the generic manifold baseline becomes highly competitive but IFM regains a small edge on all three reported metrics.
We therefore interpret the FMGG-style comparison as evidence that raw manifold geodesics recover part of the intrinsic benefit, while also exposing the tension in the narrative: IFM's advantage is not guaranteed by intrinsic geometry alone.
Rather, the strongest claim supported by these results is that quotient-aware phase alignment and the IFM path choice often help when the target ensemble has coherent projective organization, whereas smoother or benchmark-specific couplings can favor a generic geodesic path.

\begin{table*}[t]
\centering
\small
\caption{Appendix-only comparison to a generic manifold-geodesic flow-matching baseline inspired by manifold flow matching \citep{chen2023flow}, together with Euclidean FM \citep{lipman2023flowmatchinggenerativemodeling}, on $6$-qubit benchmarks and an additional $10$-qubit XXZ comparison. Lower is better.}
\label{tab:appendix_fmgg_baseline}
\begin{tabular}{llccc}
\toprule
Benchmark & Method & MMD $\downarrow$ & $\Delta_{\mathrm{obs}} \downarrow$ & Ent.~$W_1 \downarrow$ \\
\midrule
Single-cluster ($6$ qubits) & Euclidean FM & $6.0216 \times 10^{-3}$ & $2.1572 \times 10^{-3}$ & $8.1543 \times 10^{-3}$ \\
Single-cluster ($6$ qubits) & Generic geodesic FM~\citep{chen2023flow} & $2.5475 \times 10^{-4}$ & $2.4098 \times 10^{-3}$ & $5.2252 \times 10^{-3}$ \\
Single-cluster ($6$ qubits) & Geodesic IFM (ours variant) & $2.4933 \times 10^{-4}$ & $2.4820 \times 10^{-3}$ & $4.3428 \times 10^{-3}$ \\
Single-cluster ($6$ qubits) & IFM (ours) & $2.6280 \times 10^{-4}$ & $2.3388 \times 10^{-3}$ & $4.6615 \times 10^{-3}$ \\
\addlinespace
TFIM ($6$ qubits) & Euclidean FM & $3.1681 \times 10^{-2}$ & $5.6582 \times 10^{-2}$ & $6.2842 \times 10^{-2}$ \\
TFIM ($6$ qubits) & Generic geodesic FM~\citep{chen2023flow} & $1.8038 \times 10^{-2}$ & $5.3284 \times 10^{-2}$ & $1.2640 \times 10^{-1}$ \\
TFIM ($6$ qubits) & IFM (ours) & $1.7825 \times 10^{-3}$ & $1.0956 \times 10^{-2}$ & $2.5062 \times 10^{-2}$ \\
\addlinespace
Eq. bimodal ($6$ qubits) & Euclidean FM & $2.1861 \times 10^{-2}$ & $1.0823 \times 10^{-1}$ & $2.4905 \times 10^{-2}$ \\
Eq. bimodal ($6$ qubits) & Generic geodesic FM~\citep{chen2023flow} & $8.1550 \times 10^{-3}$ & $6.5664 \times 10^{-2}$ & $2.8592 \times 10^{-2}$ \\
Eq. bimodal ($6$ qubits) & IFM (ours) & $6.2495 \times 10^{-3}$ & $6.4122 \times 10^{-2}$ & $2.9000 \times 10^{-2}$ \\
\addlinespace
Trimodal ($6$ qubits) & Euclidean FM & $1.3397 \times 10^{-2}$ & $4.2138 \times 10^{-2}$ & $6.5884 \times 10^{-2}$ \\
Trimodal ($6$ qubits) & Generic geodesic FM~\citep{chen2023flow} & $6.1315 \times 10^{-4}$ & $2.3808 \times 10^{-2}$ & $5.3652 \times 10^{-2}$ \\
Trimodal ($6$ qubits) & IFM (ours) & $3.6621 \times 10^{-4}$ & $2.0026 \times 10^{-2}$ & $5.3195 \times 10^{-2}$ \\
\addlinespace
Spin-coherent ($6$ qubits) & Euclidean FM & $1.6100 \times 10^{-2}$ & $1.1463 \times 10^{-2}$ & $1.5728 \times 10^{-1}$ \\
Spin-coherent ($6$ qubits) & Generic geodesic FM~\citep{chen2023flow} & $1.1341 \times 10^{-2}$ & $9.1691 \times 10^{-3}$ & $1.6435 \times 10^{-1}$ \\
Spin-coherent ($6$ qubits) & IFM (ours) & $4.7601 \times 10^{-3}$ & $6.5775 \times 10^{-3}$ & $2.0694 \times 10^{-1}$ \\
\addlinespace
XXZ ($6$ qubits) & Euclidean FM & $1.1880 \times 10^{-2}$ & $1.0420 \times 10^{-2}$ & $2.7554 \times 10^{-2}$ \\
XXZ ($6$ qubits) & Generic geodesic FM~\citep{chen2023flow} & $9.5237 \times 10^{-3}$ & $9.9533 \times 10^{-3}$ & $1.0297 \times 10^{-2}$ \\
XXZ ($6$ qubits) & IFM (ours) & $1.2254 \times 10^{-2}$ & $4.8443 \times 10^{-2}$ & $9.1078 \times 10^{-2}$ \\
\addlinespace
XXZ ($10$ qubits) & Euclidean FM & $1.3065 \times 10^{-2}$ & $2.2184 \times 10^{-2}$ & $4.2252 \times 10^{-2}$ \\
XXZ ($10$ qubits) & Generic geodesic FM~\citep{chen2023flow} & $6.4518 \times 10^{-4}$ & $4.3118 \times 10^{-3}$ & $2.5516 \times 10^{-2}$ \\
XXZ ($10$ qubits) & IFM (ours) & $3.5211 \times 10^{-4}$ & $3.9172 \times 10^{-3}$ & $2.5046 \times 10^{-2}$ \\
\bottomrule
\end{tabular}
\end{table*}

\section{Limitations and Future Work}

Our current empirical study is carried out in the dense pure-state regime with deliberately interpretable benchmark families.
The benchmark suite is designed to isolate geometry, phase alignment, multimodality, and high-dimensional projective effects under controlled conditions, extending intrinsic quantum generative comparisons beyond the common $6$-qubit setting to structured $10$-qubit and selected $12$-qubit tests.
Natural empirical extensions include ensembles obtained from tomography pipelines, variational-circuit trajectories, simulated or experimental quantum devices, task-level QML workloads, larger many-body Hamiltonian families, and compressed representations such as tensor networks or shadows.
The framework also focuses on pure states on $\mathbb{CP}^{d-1}$; mixed-state, density-matrix, and hybrid pure-state/density-matrix transport are promising directions for extending the same intrinsic-flow viewpoint.
On the theory side, the endpoint and stability results suggest several next steps, including sharper statistical analyses, dimension-aware bounds, and connections between intrinsic flow matching, probability currents, and control-theoretic transport.

\section{Declaration of LLM usage}
LLMs were used for wording and sentence polishing.


\newpage

\section*{NeurIPS Paper Checklist}

\begin{enumerate}

\item {\bf Claims}
    \item[] Question: Do the main claims made in the abstract and introduction accurately reflect the paper's contributions and scope?
    \item[] Answer: \answerYes{} 

\item {\bf Limitations}
    \item[] Question: Does the paper discuss the limitations of the work performed by the authors?
    \item[] Answer: \answerYes{} 

\item {\bf Theory assumptions and proofs}
    \item[] Question: For each theoretical result, does the paper provide the full set of assumptions and a complete (and correct) proof?
    \item[] Answer: \answerYes{} 

    \item {\bf Experimental result reproducibility}
    \item[] Question: Does the paper fully disclose all the information needed to reproduce the main experimental results of the paper to the extent that it affects the main claims and/or conclusions of the paper (regardless of whether the code and data are provided or not)?
    \item[] Answer: \answerYes{}

\item {\bf Open access to data and code}
    \item[] Question: Does the paper provide open access to the data and code, with sufficient instructions to faithfully reproduce the main experimental results, as described in supplemental material?
    \item[] Answer: \answerYes{} 

\item {\bf Experimental setting/details}
    \item[] Question: Does the paper specify all the training and test details (e.g., data splits, hyperparameters, how they were chosen, type of optimizer) necessary to understand the results?
    \item[] Answer: \answerYes{} 

\item {\bf Experiment statistical significance}
    \item[] Question: Does the paper report error bars suitably and correctly defined or other appropriate information about the statistical significance of the experiments?
    \item[] Answer: \answerYes{} 
    \item[] Justification: All main-text tables (Tables~\ref{tab:single_cluster_baselines}, \ref{tab:cluster_scaling}, and \ref{tab:structured_benchmarks}) report mean $\pm$ standard deviation across $10$ independent random seeds (seed indices $\{1,\ldots,10\}$), with entries shown in the form $(\mu \pm \sigma) \times 10^{c}$. For each seed we select a single checkpoint using validation MMD and evaluate MMD, $\Delta_{\mathrm{obs}}$, and Ent.~$W_1$ at that checkpoint, then aggregate across seeds. The reported error bars are unbiased sample standard deviations capturing variability across seed-controlled network initialization, data sampling, and base-prior draws. The full protocol is described in Appendix~\ref{app:exp_protocol}. The observed standard deviations are small relative to the gaps between methods and do not affect the qualitative conclusions.

\item {\bf Experiments compute resources}
    \item[] Question: For each experiment, does the paper provide sufficient information on the computer resources (type of compute workers, memory, time of execution) needed to reproduce the experiments?
    \item[] Answer: \answerYes{} 
    
\item {\bf Code of ethics}
    \item[] Question: Does the research conducted in the paper conform, in every respect, with the NeurIPS Code of Ethics \url{https://neurips.cc/public/EthicsGuidelines}?
    \item[] Answer: \answerYes{} 

\item {\bf Broader impacts}
    \item[] Question: Does the paper discuss both potential positive societal impacts and negative societal impacts of the work performed?
    \item[] Answer: \answerNA{}
    \item[] Justification: As a machine learning research paper, the work is primarily technical in nature and does not directly raise specific societal impact concerns. Therefore, we consider that it does not have notable positive or negative societal impacts to discuss.
    
\item {\bf Safeguards}
    \item[] Question: Does the paper describe safeguards that have been put in place for responsible release of data or models that have a high risk for misuse (e.g., pre-trained language models, image generators, or scraped datasets)?
    \item[] Answer: \answerNA{} 
    \item[] Justification: the paper poses no such risks.

\item {\bf Licenses for existing assets}
    \item[] Question: Are the creators or original owners of assets (e.g., code, data, models), used in the paper, properly credited and are the license and terms of use explicitly mentioned and properly respected?
    \item[] Answer: \answerYes{} 

\item {\bf New assets}
    \item[] Question: Are new assets introduced in the paper well documented and is the documentation provided alongside the assets?
    \item[] Answer: \answerYes{} 

\item {\bf Crowdsourcing and research with human subjects}
    \item[] Question: For crowdsourcing experiments and research with human subjects, does the paper include the full text of instructions given to participants and screenshots, if applicable, as well as details about compensation (if any)? 
    \item[] Answer: \answerNA{} 
    \item[] Justification: the paper does not involve crowdsourcing nor research with human subjects

\item {\bf Institutional review board (IRB) approvals or equivalent for research with human subjects}
    \item[] Question: Does the paper describe potential risks incurred by study participants, whether such risks were disclosed to the subjects, and whether Institutional Review Board (IRB) approvals (or an equivalent approval/review based on the requirements of your country or institution) were obtained?
    \item[] Answer: \answerNA{} 
    \item[] Justification: the paper does not involve crowdsourcing nor research with human subjects

\item {\bf Declaration of LLM usage}
    \item[] Question: Does the paper describe the usage of LLMs if it is an important, original, or non-standard component of the core methods in this research? Note that if the LLM is used only for writing, editing, or formatting purposes and does \emph{not} impact the core methodology, scientific rigor, or originality of the research, declaration is not required.
    \item[] Answer: \answerYes{} 

\end{enumerate}

\end{document}